\PassOptionsToPackage{table}{xcolor}
\documentclass[sigconf,balance=false]{acmart}
\usepackage{amsthm}
\usepackage{amsmath}
\usepackage{algorithm}
\usepackage{enumitem}
\usepackage{picinpar}
\usepackage{lineno}
\usepackage{graphicx}
\usepackage[table]{xcolor}
\usepackage{algorithmicx}
\usepackage[noend]{algpseudocode}
\usepackage{subfigure}
\usepackage{multirow}
\usepackage{color}
\usepackage{balance}
\usepackage{enumitem}
\usepackage{hhline}
\usepackage[normalem]{ulem}
\usepackage{booktabs}
\usepackage{wrapfig}
\usepackage{cancel}
\usepackage{hyperref}
\usepackage{makecell}
\usepackage{listings}

\lstdefinestyle{fullprompt}{
  basicstyle=\ttfamily\tiny,
  breaklines=true,
  breakatwhitespace=false,
  columns=fullflexible,
  keepspaces=true,
  showstringspaces=false,
  frame=single,
  framerule=0.4pt,
  rulecolor=\color{black!35},
  captionpos=t
}

\newtheorem{proposition}{Proposition}

\renewcommand{\vec}[1]{\ensuremath{\mathbf{#1}}}

\newcommand{\stitle}[1]{\vspace{1mm} \noindent {\bf #1}}

\newcommand{\model}{TradingMoE}

\newcommand{\best}[1]{\textbf{#1}}
\newcommand{\second}[1]{\underline{#1}}

\newcommand{\eat}[1]{}

\newcommand{\stkout}[1]{\ifmmode\text{\sout{\ensuremath{#1}}}\else\sout{#1}\fi}

\usepackage{amsmath,amsfonts,bm}

\def\eqref#1{equation~\ref{#1}}

\def\1{\bm{1}}

\ifdefined\mathsfit
\else
\DeclareMathAlphabet{\mathsfit}{\encodingdefault}{\sfdefault}{m}{sl}
\SetMathAlphabet{\mathsfit}{bold}{\encodingdefault}{\sfdefault}{bx}{n}
\fi

\usepackage{hyperref}
\usepackage{url}

\title{\model: Routing the Right Experts in Evolving Markets}

\author{Chang Zhou}
\affiliation{%
  \institution{University of Science and Technology of China}
  \city{Hefei}
  \country{China}
}
\email{zhouchang21sy@mail.ustc.edu.cn}

\author{Xingtong Yu}
\authornote{Corresponding authors.}
\affiliation{%
  \institution{The Chinese University of Hong Kong}
  \city{Hong Kong}
  \country{Hong Kong SAR, China}
}
\email{xtyu@se.cuhk.edu.hk}

\author{Minbin Huang}
\affiliation{%
  \institution{The Chinese University of Hong Kong}
  \city{Hong Kong}
  \country{Hong Kong SAR, China}
}
\email{huangminbin@link.cuhk.edu.hk}

\author{Zhennan Wu}
\affiliation{%
  \institution{The University of Tokyo}
  \city{Tokyo}
  \country{Japan}
}
\email{swwzn714@gmail.com}

\author{Yuan Fang}
\affiliation{%
  \institution{Singapore Management University}
  \city{Singapore}
  \country{Singapore}
}
\email{yfang@smu.edu.sg}

\author{Hong Cheng}
\affiliation{%
  \institution{The Chinese University of Hong Kong}
  \city{Hong Kong}
  \country{Hong Kong SAR, China}
}
\email{ hcheng@se.cuhk.edu.hk}

\author{Xinming Zhang}
\authornotemark[1]
\affiliation{%
  \institution{University of Science and Technology of China}
  \city{Hefei}
  \country{China}
}
\email{xinming@ustc.edu.cn}

\renewcommand{\shortauthors}{Zhou et al.}

\begin{document}

\begin{abstract}
Large language models (LLMs) have shown strong potential for financial analysis and trading, but direct trading remains challenging because the predictive capabilities required can vary across assets, decision fields, and market conditions. Existing LLM-based trading systems either coordinate human-defined external experts or adopt conventional internal Mixture-of-Experts (MoE) routers that do not directly evaluate how individual experts contribute to trading decisions.
Moreover, these routers receive no direct signal indicating when an inactive
expert has become more suitable as market conditions change.
We find that native router scores poorly reflect how much individual experts
improve trading decisions, frequently leaving better alternatives unselected.
We further reveal that token-specific expert usefulness exhibits a compact
low-dimensional structure.
Based on these findings, we propose \model{}, a trading-oriented sparse MoE
that augments a frozen dense LLM with lightweight residual experts.
We introduce a Query--Key router that represents the expertise required by
each token under the current market context as a low-dimensional query and
matches it with learnable expert keys.
We further propose a sparse expert selection update mechanism that samples a
few inactive experts during training and estimates whether they should replace
the weakest expert in the current Top-$k$ route.
This mechanism enables the router to update expert selection as market
conditions change while preserving sparse computation.
Experiments against 22 baselines on stock and cryptocurrency markets show that
\model{} improves cumulative return over the best-performing baselines by $30.89\%$ and $30.7\%$, respectively, while exceeding the
corresponding buy-and-hold benchmarks by $37.81\%$ and $80.17\%$ percentage points.
Rolling paper-trading experiments further demonstrate that its advantage
persists under forward-only deployment. Code is available at \url{https://anonymous.4open.science/r/TradingMoE-DC52}.
\end{abstract}

\ccsdesc[500]{Computing methodologies~Neural networks}
\ccsdesc[500]{Information systems~Data mining}
\ccsdesc[500]{Applied computing~Economics}

\keywords{Large language models, ai4trading, mixture-of-experts}

\maketitle
\section{Introduction}\label{sec.intro}

Artificial intelligence has been widely applied to trading, supporting tasks such as return forecasting, asset ranking, and portfolio optimization \citep{feng2019temporal,xu2021hist,li2024master,jiang2017deep,wang2019alphastock,ye2020reinforcement,liu2021finrl,wang2021commission}. More recently, financial large language models (LLMs) have enabled the joint processing of heterogeneous market information \citep{wu2023bloomberggpt,yang2023fingpt}. Existing methods use LLMs either to extract signals and support market analysis
\citep{yang2019leveraging,ding2023integrating,wang2024llmfactor,guo2024finetuning} or to directly generate trading decisions \citep{xiao2024tradingagents,zhang2024finagent,xiong2025flagtrader}.
However, these approaches typically rely on a single dense model or a fixed agentic workflow to process all market states and decision components. Such uniform computation is poorly suited to heterogeneous and non-stationary financial markets, where the predictive capabilities required can vary substantially across assets, decision fields, and market regimes. This limitation motivates methods that can dynamically allocate specialized modeling capacity according to the current trading context.

\begin{figure*}[t]
    \centering
    \includegraphics[width=\linewidth]{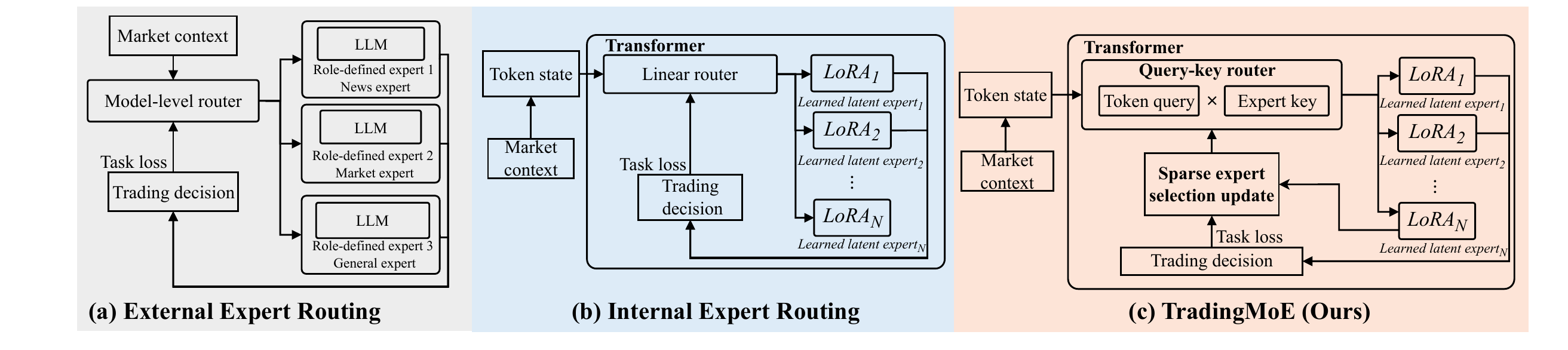}
    \vspace*{-8mm}
    \caption{Comparison of expert routing methods.}
    \label{fig:motivation}
\end{figure*}

Existing LLM-based trading systems commonly realize specialization through external expert routing, as shown in Fig.~\ref{fig:motivation}(a), where an LLM coordinates human-defined external modules, such as news analysts and price predictors \citep{ding2024tradexpert,liu2025llmoe,chen2025mmdrex}. However, such systems rely on predefined expert roles and typically route an entire market state or trading task to a few model-level experts, making it difficult to adapt expert usage across assets, information tokens, and decision fields. 
Internal expert routing, commonly instantiated through sparse \emph{Mixture-of-Experts} (MoE), provides a finer-grained alternative by moving expert routing inside the language model, as shown in Fig.~\ref{fig:motivation}(b), in which a learned router activates only a small subset of experts for each token \citep{shazeer2017outrageously,lepikhin2020gshard,fedus2022switch,jiang2024mixtral}.
However, conventional routers are trained mainly through the task loss and routing regularizers, which do not directly reveal each expert's contribution or how unselected experts would perform. Moreover, each expert's contribution to the trading decision may change as market conditions evolve.
This leaves a central problem unresolved:
\emph{how to evaluate each expert's contribution to trading decisions and track its changes as market conditions evolve?}
This problem is non-trivial due to two key challenges.

First, \emph{how can we quantify each expert's contribution to a trading decision?}  Conventional routers assign expert scores from token hidden states \citep{fedus2022switch,lepikhin2020gshard}.  These scores are only implicit proxies for expert suitability: the training objective does not directly verify whether a higher-scored expert contributes more to reducing the trading-decision loss.  
We therefore assess routing quality through controlled expert replacement. For each held-out decision token, we replace a selected expert with an unselected alternative while keeping the number of active experts unchanged, and measure the resulting change in the trading-decision loss. If router scores faithfully reflect expert contribution, they should rank experts consistently with these replacement gains.
Instead, naive router (refers to conventional MoE models \citep{liu2024deepseekv2}) scores have a Pearson correlation \citep{pearson1895vii} of only $-0.015$ with the measured gains, and $66.76\%$ of decision tokens leave at least one better expert unselected.
This substantial mismatch motivates a measure of expert contribution. We define counterfactual expert credit using a standard first-order approximation \citep{molchanov2019importance,michel2019sixteen}. It estimates how replacing a selected expert with an alternative would change the trading-decision loss; a higher credit means that the alternative is expected to reduce the loss more.  We aggregate these credits into a token--expert matrix and analyze its structure across two pretrained MoE-LLMs \citep{muennighoff2025olmoe,liu2024deepseekv2}. As shown in Table~\ref{tab:intro_lowrank_credit},  the matrices exhibit a pronounced low-rank structure \citep{eckart1936approximation}: rank-16 reconstructions retain $74.2\%$ and $77.9\%$ of the credit energy and identify higher-credit experts than the native routers on over $98\%$ of decision tokens. This finding suggests that expert suitability can be represented in a compact latent space. Motivated by this structure, we propose a \emph{Query--Key} router. For each decision token, the router constructs a low-dimensional trading-demand query that represents the expertise required by the token under the current market context. Each expert is associated with a learnable expert key, and computes routing scores through query--key matching.

\begin{table}[t]
\centering
\caption{Low-rank analysis of token--expert counterfactual credit.}
\vspace*{-10pt}
\label{tab:intro_lowrank_credit}
\small
\resizebox{\linewidth}{!}{
\begin{tabular}{lcccccc}
\toprule
& \multicolumn{3}{c}{Credit Energy Retained} & \multicolumn{3}{c}{Better-than-Native Route Rate} \\
\cmidrule(lr){2-4}\cmidrule(lr){5-7}
MoE backbone & $r=1$ &$r=4$ & $r=16$ & $r=1$ &\hspace{10pt}$r=4$ & $r=16$ \\
\midrule
OLMoE & 0.185 & 0.373 & 0.742 & 0.643& \hspace{10pt}0.812 & 0.983 \\
DeepSeek-V2-Lite & 0.117 & 0.359 & 0.779 & 0.734& \hspace{10pt}0.875 & 0.986 \\
\bottomrule
\end{tabular}
}
\parbox{\linewidth}{\footnotesize $r$ is the reconstruction rank. Better-than-Native Route Rate is the rate of decision tokens for which the experts selected using this compact representation have higher total credit than those selected by the original router.}
\vspace{-20pt}
\end{table}

Second, \emph{how can the router learn when its expert selection should change across market conditions?}
Expert contribution  is inherently market-dependent: an expert that is effective in one regime may become ineffective after a regime shift or event shock \citep{ang2012regime,lo2004adaptive}. As market evolves, the router must therefore learn not only which experts are currently useful, but also when an inactive expert has become a better alternative and should enter the route.
Conventional sparse routing exposes the task loss only for the experts activated at each token, leaving the current contribution of inactive experts unobserved \citep{shazeer2017outrageously,fedus2022switch}.
In this work, we propose a \textit{sparse expert selection update} mechanism. At each decision-value token, we sample a small number of inactive experts as challengers and compare them with the lowest-scored active expert in the current Top-$k$ route. Their first-order relative credits estimate whether a challenger would reduce the trading-decision loss more than the lowest-scored active expert and should therefore replace it.

In summary, our contributions are fourfold.
First, we uncover that native routers often fail to select the experts that best reduce the trading-decision loss, and token--expert counterfactual credits exhibit a pronounced approximately low-rank structure.
Second, we propose \model{}, a trading-oriented sparse internal expert routing MoE that combines query-key router for modeling expert contribution with a sparse expert selection update mechanism for adapting the route as market context changes.
Third, we provide theoretical proof for the proposed expert selection update mechanism, showing that inactive-expert sampling yields an unbiased estimate of the update over all inactive experts and that the routing-margin update is consistent with the local loss reduction induced by expert replacement.
Fourth,  extensive experiments on stock, crypto, and rolling paper-trading  demonstrate substantial improvements of \model{} over 22 baselines.

\section{Related Work}\label{sec.related}

\noindent\textbf{Conventional AI methods for trading.}
Traditional machine-learning methods for trading typically rely on structured market signals, such as price-volume features and technical indicators to predict future returns, movement directions, or stock rankings \citep{feng2019temporal,wang2019alphastock}. 
Deep learning methods further improve trading models by learning temporal patterns, cross-sectional dependencies, and market-aware representations from financial data \citep{xu2021hist,li2024master}. 
Reinforcement learning approaches formulates trading as a sequential decision-making problem and optimize trading actions or portfolio weights under transaction costs and risk-aware objectives \citep{jiang2017deep,wang2019alphastock,ye2020reinforcement,liu2021finrl,wang2021commission}. 
However, these approaches are usually built on fixed numerical inputs and specialized task objectives, limiting their ability to incorporate heterogeneous financial information.

\noindent\textbf{LLMs-based methods for trading.}
Recent studies adapt large language models to financial scenarios, including financial text understanding, document analysis, numerical reasoning and market forecasting\citep{wu2023bloomberggpt,yang2023fingpt,tang2023finentity,mukherjee2022ectsum,chen2021finqa,chen2022convfinqa,zhu2025fincast,yang2019leveraging,ding2023integrating,wang2024llmfactor,guo2024finetuning}. More recent LLM-based trading systems move toward direct decision-making through agentic workflows, where LLMs act as analysts, planners, or policy agents in trading pipelines \citep{xiao2024tradingagents,zhang2024finagent,xiong2025flagtrader}.
However, these systems typically rely on a single dense model or a fixed agentic workflows, providing limited flexibility in how
specialized capabilities are invoked during decision generation.


\noindent\textbf{Expert routing for trading.}
External expert routing decomposes trading analysis and decision-making into
human-defined expert modules and coordinates them at the task or model level
\citep{chen2025mmdrex,ding2024tradexpert,liu2025llmoe}.
While this paradigm integrates heterogeneous expertise, its predefined expert
roles and coarse routing granularity limit adaptation across assets and
decision components.
Internal expert routing instead performs token-level expert selection within
Transformer layers, enabling finer-grained specialization
\citep{shazeer2017outrageously,lepikhin2020gshard,fedus2022switch,jiang2024mixtral}.
However, conventional internal routers score experts from token hidden states
without explicit supervision on whether an unselected expert would reduce the
trading-decision loss.
Our work addresses this gap by designing and supervising internal expert
routing for direct generation of trading decisions.

\section{Preliminaries}\label{sec.preliminary}

\noindent\textbf{Daily trading task formulation.}
We formulate trading task as a decision-generation problem. For each trading day $n$, let $\tau$ denote the predefined decision time.
The model uses only market information available before $\tau$ to generate the trading decision for that day. For notational simplicity, we omit the day index $n$ in the following.
The input is a set of candidate assets $\mathcal{X}=\{x_{1},x_{2},\cdots\}$. Asset $i$ is represented as a structured sequence $x_{i}$ comprising its
sector, market capitalization, five-day open, high, low, close, and volume (OHLCV) \cite{garman1980estimation} history, indicators \cite{wilder1978new}, and timestamped asset-related news. Details of the construction of the input is shown in Appendix~\ref{app:setup-data}. $\mathcal{X}$ is fed into a trading model, which generates a set of trading decisions $\mathcal{Y}=\{y_{1},y_{2},\cdots,y_{M}\}$. For asset $i$, the decision is represented as $y_i=(a_i,q_i),$ where $a_i\in\{\texttt{long},\texttt{short},\texttt{hold}\} $ denotes the trading action, and $q_i\geq 0$ denotes the corresponding position size.
A \texttt{long} action allocates $q_i$ to a long position, whereas a \texttt{short} action allocates $q_i$ to a short position. For a \texttt{hold} action, $q_i$ preserves the current position rather than opening a new position.

\noindent\textbf{Internal expert routing.}
Here we introduce a conventional internal expert routing architecture \cite{shazeer2017outrageously,fedus2022switch}.
Specifically, the market input $\mathcal X$  is first serialized, tokenized, and mapped to an embedding matrix:
\begin{equation}
\mathbf{H}^{0}
=
\mathtt{Embed}
\left(
\mathtt{Tokenizer}
\left(
\mathtt{Serialize}(\mathcal X)
\right)
\right).
\label{eq:input-token-embedding}
\end{equation}
where $\mathtt{Serialize}(\cdot)$ converts the structured multi-asset market
state into the textual sequence. $\mathtt{Tokenizer}(\cdot)$ denotes the tokenizer of a language model, which converts the market input into a sequence of $T$ tokens. $\mathtt{Embed}(\cdot)$ is the token-embedding layer, mapping each token to a $d$-dimensional embedding. $\vec{H}^0$ serves as the input for the transformer. 
Consider a routed Transformer layer $\ell$ with input
$\mathbf{H}^{\ell-1}\in\mathbb{R}^{T\times d}$.
Before expert routing, the attention sub-layer computes
\begin{equation}
\overline{\mathbf{H}}^{\ell}
=
\mathbf{H}^{\ell-1}
+
\mathtt{Attn}^{\ell}
\left(
\mathtt{LN}^{\ell}
\left(\mathbf{H}^{\ell-1}\right)
\right),
\label{eq:pre-routing-representation}
\end{equation}
where $\mathtt{Attn}^{\ell}(\cdot)$ denotes the multi-head self-attention module and $\mathtt{LN}^{\ell}$ denotes the corresponding layer-normalization module. Let $\overline{\mathbf{h}}_{t}^{\ell}$ denote the $t$-th row of $\overline{\mathbf{H}}^{\ell}$.
At layer $\ell$, the routing module contains $E$ lightweight low-rank
experts \cite{hu2022lora,tian2024hydralora}.
All experts share a trainable down-projection
$\mathbf{A}^{\ell}\in\mathbb{R}^{d_{\mathrm{lr}}\times d}$, while each
expert $i$ has an expert-specific up-projection
$\mathbf{B}_{i}^{\ell}\in\mathbb{R}^{d\times d_{\mathrm{lr}}}$.
The output of expert $i$ for token $t$ is defined as
\begin{equation}
e_{i}^{\ell}
\left(
\overline{\mathbf{h}}_{t}^{\ell}
\right)
=
\frac{\lambda_{\mathrm{lr}}}{d_{\mathrm{lr}}}
\mathbf{B}_{i}^{\ell}
\mathbf{A}^{\ell}
\overline{\mathbf{h}}_{t}^{\ell},
\label{eq:low-rank-expert}
\end{equation}
where $d_{\mathrm{lr}}$ denotes the low-rank dimension \cite{hu2022lora} and
$\lambda_{\mathrm{lr}}$ is the corresponding scaling factor \cite{hu2022lora}.
The expert parameters at layer $\ell$ are collected as
\begin{equation}
\Phi^{\ell}
=
\left\{
\mathbf{A}^{\ell},
\mathbf{B}_{1}^{\ell},
\ldots,
\mathbf{B}_{E}^{\ell}
\right\}.
\label{eq:layer-expert-parameters}
\end{equation}
For each token $t$, a layer-specific scoring multilayer perceptron produces routing scores for
all experts:
\begin{equation}
\mathbf{s}_{t}^{\ell}
=
\mathtt{Score}^{\ell}
\left(
\overline{\mathbf{h}}_{t}^{\ell};
\boldsymbol{\theta}^{\ell}
\right)
\in\mathbb{R}^{E},
\label{eq:expert-routing-score}
\end{equation}
where $\mathtt{Score}^{\ell}(\cdot)$ is the scoring network at layer
$\ell$, with parameter $\boldsymbol{\theta}^{\ell}$.
The $i$-th element $s_{t,i}^{\ell}$ of $\mathbf{s}_{t}^{\ell}$ denotes the
routing score assigned to expert $e_i^\ell$.
The router activates the $K$ highest-scoring experts:
\begin{equation}
\mathcal{I}_{t}^{\ell}
=
\mathtt{TopK}
\left(
\mathbf{s}_{t}^{\ell}
\right),
\label{eq:topk-expert-selection}
\end{equation}
where $\mathcal{I}_{t}^{\ell}\subseteq\{1,\ldots,E\}$ is the index set of
selected experts and $k$ is a hyperparameter that controls the routing sparsity.
The routing weights over the selected experts are computed by
\begin{equation}
\boldsymbol{\alpha}_{t}^{\ell}
=
\mathtt{Softmax}
\left(
\left[
s_{t,i}^{\ell}
\right]_{i\in\mathcal{I}_{t}^{\ell}}/\eta
\right),
\label{eq:expert-routing-weight}
\end{equation}
where $\mathtt{Softmax}(\cdot)$ \cite{bridle1990probabilistic} normalizes the routing scores over the selected
Top-$K$ experts, $\eta>0$ is the temperature parameter and $\alpha_{t,i}^{\ell}$ denotes the resulting weight of expert $i$.
The selected expert outputs are aggregated into a sparse expert residual:
\begin{equation}
\mathbf{r}_{t}^{\ell}
=
\sum_{i\in\mathcal{I}_{t}^{\ell}}
\alpha_{t,i}^{\ell}\cdot
e_{i}^{\ell}(\overline{\mathbf{h}}_{t}^{\ell}).
\label{eq:internal-routing-output}
\end{equation}

The expert residual operate in parallel with the original feed-forward network of the Transformer layer.
The output representation of token $t$ is therefore
\begin{equation}
\mathbf{h}_{t}^{\ell}
=
\mathtt{FFN}^{\ell}
\left(
\overline{\mathbf{h}}_{t}^{\ell}
\right)
+
\mathbf{r}_{t}^{\ell},
\label{eq:routed-layer-output}
\end{equation}
where $\mathtt{FFN}^{\ell}(\cdot)$ denotes the original feed-forward network
at layer $\ell$.
The complete output of the routed layer is $\mathbf{H}^{\ell}
=
\left[
\mathbf{h}_{1}^{\ell};
\mathbf{h}_{2}^{\ell};
\ldots;
\mathbf{h}_{T}^{\ell}
\right]
\in\mathbb{R}^{T\times d}.
$
Let $\mathcal{L}$ denote the set of routed Transformer layers.
We collect the expert parameters and scoring-network parameters across all
routed layers as
$
\Phi
=
\left\{
\Phi^{\ell}
\right\}_{\ell\in\mathcal{L}}$ and
$\Theta
=
\left\{
\boldsymbol{\theta}^{\ell}
\right\}_{\ell\in\mathcal{L}},
$ respectively.
The output $\mathbf{H}^{\ell}$ is passed to the subsequent Transformer layer,
and the final hidden representations are decoded to generate
the structured trading decisions $\mathcal{Y}$.


\stitle{Training objective.}
Given a market input $\mathcal X$, the model defines an autoregressive
distribution
$p_{\Phi,\Theta}(\mathcal Y\mid\mathcal X)$
over structured trading decisions, where $\Phi$ and $\Theta$ denote the
trainable expert and router parameters, respectively.
Let $\mathcal Y^{\star}$ denote the corresponding ground-truth trading
decision, constructed as described in
Appendix~\ref{app:ground-truth-construction}.

The training objective minimizes the negative log-likelihood of the
ground-truth decision:
\begin{equation}
\mathcal L_{\mathrm{task}}(\Phi,\Theta)
=
-
\log
p_{\Phi,\Theta}
\left(
\mathcal Y^{\star}
\mid
\mathcal X
\right).
\label{eq:decision-level-loss}
\end{equation}

To compute this objective autoregressively, we serialize and tokenize
$\mathcal Y^{\star}$ into
\begin{equation}
\mathbf z^{\star}
=
\mathtt{Tokenizer}
\left(
\mathtt{Serialize}
\left(
\mathcal Y^{\star}
\right)
\right)
=
\left(
z_1^{\star},
\ldots,
z_{|\mathbf z^{\star}|}^{\star}
\right).
\label{eq:target-token-sequence}
\end{equation}
Under teacher forcing, the preceding ground-truth tokens
$z_{<j}^{\star}
=
(z_1^{\star},\ldots,z_{j-1}^{\star})$
are provided when predicting $z_j^{\star}$.
Accordingly, Eq.~\ref{eq:decision-level-loss} is computed as
\begin{equation}
\mathcal L_{\mathrm{task}}(\Phi,\Theta)
=
-
\sum_{j=1}^{|\mathbf z^{\star}|}
\log
p_{\Phi,\Theta}
\left(
z_j^{\star}
\mid
\mathcal X,
z_{<j}^{\star}
\right).
\label{eq:task-loss}
\end{equation}
Training minimizes this loss with respect to $\Phi$ and $\Theta$, while the
parameters of the pretrained language-model backbone remain frozen.
\section{Proposed Method}\label{sec.method}

\begin{figure*}[t]
\centering
\includegraphics[width=1\textwidth]{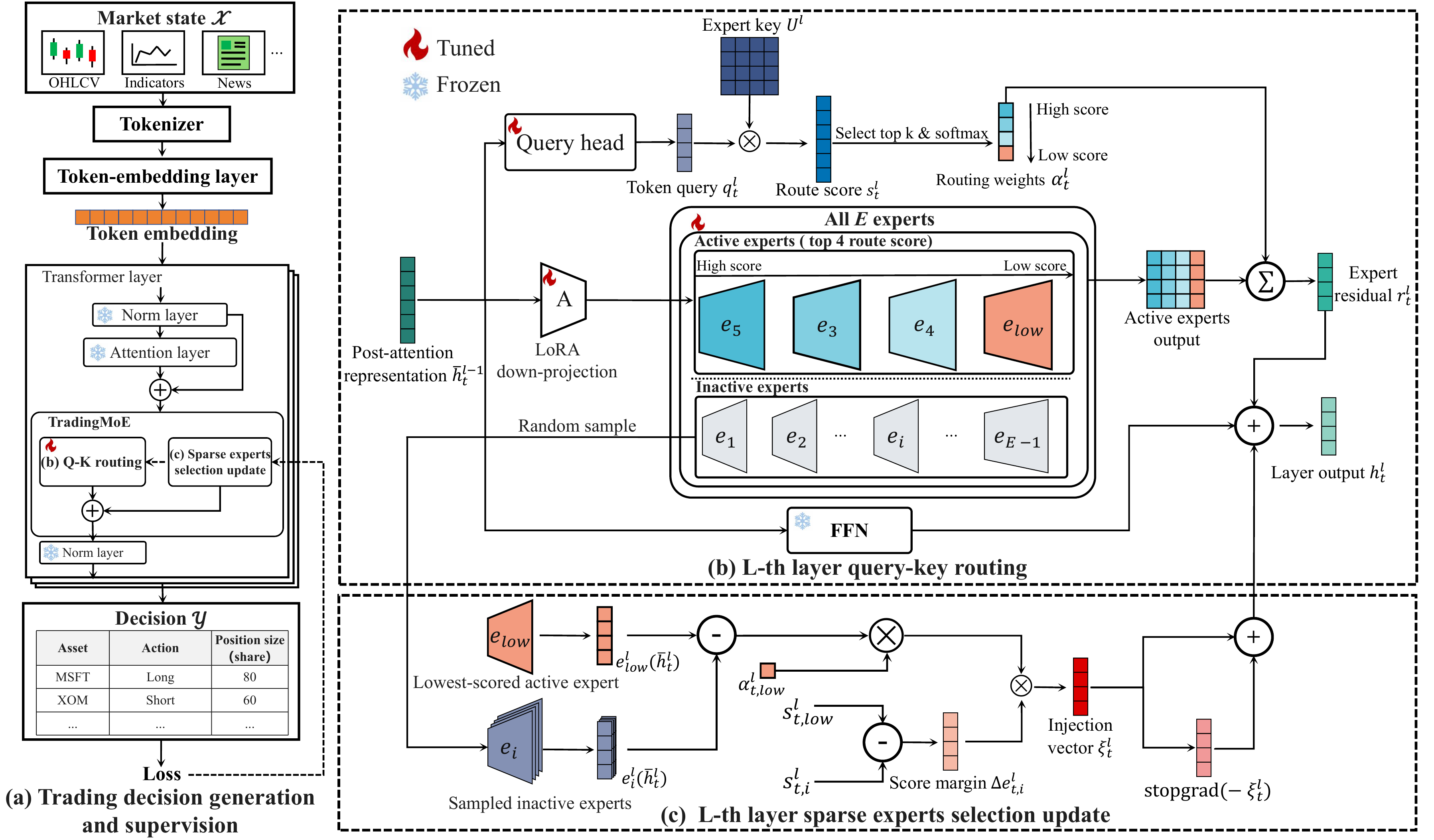}
\vspace{-8mm}
\caption{Overview of \model{}.}
\label{fig:method_framework}
\end{figure*}

\subsection{Overview}\label{sec.method.overview}

We illustrate the overall framework of \model{} in Fig.~\ref{fig:method_framework}. 
As shown in Fig.~\ref{fig:method_framework}(a), the structured market state
is processed by a frozen LLM backbone whose routed Transformer layers are
augmented with token-level sparse residual experts.
The final hidden representations are decoded into structured trading
decisions.
\model{} contains two key components.
The Query--Key router in Fig.~\ref{fig:method_framework}(b) maps each
post-attention token representation to a low-dimensional trading-demand query
and matches it with learnable expert keys to select the Top-$k$ experts.
The sparse expert selection update in Fig.~\ref{fig:method_framework}(c)
samples inactive experts, compares them with the lowest-scored selected
expert, and uses their replacement information to update the router within
the same backward pass.

\subsection{Query--Key Routing}\label{sec.method.qr}
The low-rank structure of token--expert credit suggests that
expert suitability can be captured through low-dimensional interactions
between token requirements and expert representations.
Accordingly, we introduce a Query--Key router to parameterize
token--expert routing scores.

\stitle{Low-rank query--key matching.}
At routed layer $\ell$, the router takes the post-attention representation
$\overline{\mathbf h}_{t}^{\ell}\in\mathbb{R}^{d}$ of token $t$ as input.
A layer-specific query network maps it to a low-dimensional
trading-demand query:
\begin{equation}
\mathbf q_{t}^{\ell}
=
f^{\ell}
\left(
\overline{\mathbf h}_{t}^{\ell};
\boldsymbol{\theta}^{\ell}
\right)
\in\mathbb R^{d_{\mathrm q}},
\label{eq:qk-query}
\end{equation}
where $f^{\ell}$ denotes the query network, implemented as a MLP. $\boldsymbol{\theta}^{\ell}$ denotes its parameters, and
$d_{\mathrm q}$ is the routing dimension.
The query is conditioned on both the current token and the market context
encoded in its hidden representation, and is intended to capture the
expertise required for processing that token.
Each expert $i$ at layer $\ell$ is associated with a learnable key $\mathbf u_{i}^{\ell}\in\mathbb R^{d_\vec{q}}$. The routing score between token $t$ and expert $i$ is computed as
\begin{equation}
s_{t,i}^{\ell}
=
\left(
\mathbf q_{t}^{\ell}
\right)^{\top}
\mathbf u_{i}^{\ell}.
\label{eq:qk-score}
\end{equation}
A higher score indicates a stronger match between the expertise required by the token and the corresponding expert.

To make the low-rank parameterization explicit, let
$\mathbf Q^\ell\in\mathbb R^{T\times d_{\mathrm q}}$ collect the queries of
all tokens and
$\mathbf U^\ell\in\mathbb R^{E\times d_{\mathrm q}}$ collect the keys of all
experts. The complete query--key score matrix is
\begin{equation}
\mathbf S^\ell
=
\mathbf Q^\ell
\left(
\mathbf U^\ell
\right)^\top
\in\mathbb R^{T\times E}.
\label{eq:qk-score-matrix}
\end{equation}
Therefore,
$\operatorname{rank}(\mathbf S^\ell)\leq d_{\mathrm q}$,
providing a compact parameterization consistent with the observed low-rank credit structure.

\stitle{Experts routing and aggregation.}
The router then selects the $K$ highest-scoring experts
$\mathcal I_{t}^{\ell}$ and computes their normalized routing weights
$\boldsymbol{\alpha}_{t}^{\ell}$ following
Eqs.~\ref{eq:topk-expert-selection} and
\ref{eq:expert-routing-weight}, respectively.
Each selected expert applies its low-rank transformation to the
post-attention token representation
$\overline{\mathbf h}_{t}^{\ell}$, as defined in
Eq.~\ref{eq:low-rank-expert}.
The resulting expert outputs are weighted by
$\boldsymbol{\alpha}_{t}^{\ell}$ and aggregated into the sparse residual
$\mathbf r_t^\ell$ according to Eq.~\ref{eq:internal-routing-output}.
This residual is combined with the output of the frozen feed-forward network and the
post-attention representation to produce the final token representation
$\mathbf h_t^\ell$, as defined in Eq.~\ref{eq:routed-layer-output}.
Collecting all token representations yields the routed-layer output
$\mathbf H^\ell$, which is passed to the subsequent Transformer layer.

\subsection{Sparse Expert Selection Update}\label{sec.method.credit}
In naive top-$k$ routing, the task loss backpropagates only through the selected expert outputs and their routing scores. Inactive experts therefore receive no task-dependent signal indicating whether their routing scores should be increased. To address this issue, we propose a sparse expert selection update mechanism.

\stitle{Detached routing margin.}
We first randomly sample $m$ candidates from inactive experts, denote their index set by $\hat{\mathcal{I}}_t^\ell$.
Then we define the lowest-scored expert in the current top-$k$ route as $e_{low}^\ell(\cdot)$.
For each sampled inactive expert
$i\in\hat{\mathcal I}_t^\ell$, we compare its routing score with that of
$e_{low}^\ell(\cdot)$.
To prevent the resulting supervision from propagating to the attention
sub-layer, we compute the score margin using a detached token representation:
\begin{equation}
\Delta{s}_{t,i}^{\ell}
=
f^{\ell}
\left(
\mathtt{stopgrad}
\left(
\overline{\mathbf h}_{t}^{\ell}
\right);
\boldsymbol{\theta}^{\ell}
\right)^{\top}
\left(
\mathbf u_{i}^{\ell}
-
\mathbf u_{low}^{\ell}
\right),
\label{eq:detached-router-score}
\end{equation}
where $\mathtt{stopgrad}(\cdot)$ preserves its input during the
forward pass but stops gradient propagation during backpropagation. $\mathbf u_{low}^{\ell}$ denotes the key of
$e_{low}^{\ell}(\cdot)$.
A positive $\Delta {s}_{t,i}^{\ell}$ indicates that the sampled inactive expert receives a higher routing score than the lowest-scored active expert.

\noindent\textbf{Same-step router update.}
We combine the detached routing margins with the corresponding expert-output differences:
\begin{equation}
\boldsymbol{\xi}_{t}^{\ell}
=
\frac{\lambda_{\mathrm{sel}}}{m}
\sum_{i\in\hat{\mathcal I}_{t}^{\ell}}
\Delta s_{t,i}^{\ell}
\mathtt{stopgrad}(\alpha_{t,\mathrm{low}}^{\ell}
\left[
e_i^\ell
\left(
\overline{\mathbf h}_{t}^{\ell}
\right)
-
e_{low}^\ell
\left(
\overline{\mathbf h}_{t}^{\ell}
\right)
\right]),
\label{eq:credit-injection}
\end{equation}
where $\lambda_{\mathrm{sel}}>0$ controls the strength of the expert selection update.
We inject this information into the routed-layer output through
\begin{equation}
\mathbf h_{t}^{\ell}
=
\mathtt{FFN}^{\ell}
\left(
\overline{\mathbf h}_{t}^{\ell}
\right)
+
\mathbf r_{t}^{\ell}
+
\boldsymbol{\xi}_{t}^{\ell}
-
\mathtt{stopgrad}
\left(
\boldsymbol{\xi}_{t}^{\ell}
\right).
\label{eq:zero-forward-injection}
\end{equation}
As Eqs.~\ref{eq:credit-injection} and \ref{eq:zero-forward-injection} cancel in the forward pass, the model prediction and the task loss in Eq.~\ref{eq:task-loss} remain unchanged.
During backpropagation, however, gradients through $\boldsymbol{\xi}_{t}^{\ell}$ update the query-network parameters $\Theta$ and expert-key matrices $\mathbf U=\{\mathbf U^\ell\}$, including the keys of sampled inactive experts. Meanwhile, the selected expert parameters $\Phi$ are updated through the
original routed branch. Training therefore minimizes
$\mathcal L_{\mathrm{task}}(\Phi,\Theta,\mathbf U)$ with respect to $\Phi$, $\Theta$, and $\mathbf U$, while keeping the pretrained backbone frozen.

\subsection{Theoretical Justification}
We theoretically prove two properties of the sparse expert selection mechanism.
First, sampling a small number of inactive experts provides an unbiased estimate of the update averaged over all inactive experts. Second, the router update adjusts the routing margin according to
whether replacing the lowest-scored selected expert with a sampled inactive expert would locally reduce the task loss.

\begin{proposition}[Unbiased inactive-expert sampling]
\label{prop:inactive-unbiased}
Let
$\boldsymbol{\psi}^{\ell}
=
\{\boldsymbol{\theta}^{\ell},\mathbf U^{\ell}\}$
denote the router parameters at layer $\ell$, and let
$\mathbf v_{t,i}^{\ell}$ denote the gradient contribution to
$\boldsymbol{\psi}^{\ell}$ obtained by evaluating inactive expert $i$.
If $\hat{\mathcal I}_t^\ell$ contains $m$ experts sampled uniformly without
replacement from $(\mathcal I_t^\ell)^c$, then
\begin{equation}
\mathbb E_{\hat{\mathcal I}_t^\ell}
\left[
\frac{1}{m}
\sum_{i\in\hat{\mathcal I}_t^\ell}
\mathbf v_{t,i}^{\ell}
\right]
=
\frac{1}{E-k}
\sum_{i\in(\mathcal I_t^\ell)^c}
\mathbf v_{t,i}^{\ell}.
\label{eq:inactive-unbiased}
\end{equation}
Therefore, the sampled router-update gradient is an unbiased estimate of the
average gradient over all inactive experts.
\end{proposition}

Sampling reduces the additional expert computation from
$O(E-k)$ to $O(m)$.
The proof is provided in
Appendix~\ref{app:inactive-sampling-proof}.

\begin{proposition}[Replacement-consistent routing-margin update]
\label{prop:routing-margin-update}
For each sampled inactive expert
$i\in\hat{\mathcal I}_t^\ell$, define the expert-output difference used in
Eq.~\ref{eq:credit-injection} as
\begin{equation}
\Delta\mathbf e_{t,i}^{\ell}
=
\alpha_{t,\mathrm{low}}^{\ell}
\left[
e_i^\ell
\left(
\overline{\mathbf h}_t^\ell
\right)
-
e_{\mathrm{low}}^\ell
\left(
\overline{\mathbf h}_t^\ell
\right)
\right].
\label{eq:theoretical-output-difference}
\end{equation}
The same-step router update induces
\begin{equation}
\left.
\frac{\partial\mathcal L_{\mathrm{task}}}
{\partial\Delta s_{t,i}^{\ell}}
\right|_{\mathrm{sel}}
=
\frac{\lambda_{\mathrm{sel}}}{m}
\left\langle
\frac{\partial\mathcal L_{\mathrm{task}}}
{\partial\mathbf h_t^\ell},
\Delta\mathbf e_{t,i}^{\ell}
\right\rangle.
\label{eq:routing-margin-gradient}
\end{equation}

Let $G_{t,i}^{\ell}(\rho)$ denote the reduction in task loss obtained by
perturbing the routed-layer output along
$\Delta\mathbf e_{t,i}^{\ell}$:
\begin{equation}
G_{t,i}^{\ell}(\rho)
=
\mathcal L_{\mathrm{task}}
\left(
\mathbf h_t^\ell
\right)
-
\mathcal L_{\mathrm{task}}
\left(
\mathbf h_t^\ell
+
\rho\Delta\mathbf e_{t,i}^{\ell}
\right).
\label{eq:replacement-loss-reduction}
\end{equation}
If the task loss is twice differentiable in a local neighborhood, then
\begin{equation}
G_{t,i}^{\ell}(\rho)
=
-
\rho
\left\langle
\frac{\partial\mathcal L_{\mathrm{task}}}
{\partial\mathbf h_t^\ell},
\Delta\mathbf e_{t,i}^{\ell}
\right\rangle
+
O
\left(
\rho^2
\left\|
\Delta\mathbf e_{t,i}^{\ell}
\right\|^2
\right).
\label{eq:replacement-first-order}
\end{equation}
Consequently, if replacing
$e_{\mathrm{low}}^\ell(\cdot)$ with $e_i^\ell(\cdot)$ is locally expected
to reduce the task loss, gradient descent increases
$\Delta s_{t,i}^{\ell}$; otherwise, it decreases the routing margin.
\end{proposition}
Note that the gradient in Eq.~\ref{eq:routing-margin-gradient} is proportional to
the first-order counterfactual expert credit used in the diagnostic analysis
of Sect.~\ref{sec.intro}.
Thus, the same-step update implicitly converts this credit into supervision
on the routing margin without introducing an additional credit loss.

The proof is provided in
Appendix~\ref{app:relative-credit-validity-proof}.

\section{Experiments}\label{sec.exp}

In this section, we evaluate \model{} on stock and cryptocurrency markets and further assess its prospective performance through rolling paper trading.

\begin{table*}[!t]
\centering
\caption{Portfolio performance on Stock.}
\label{tab:stock_main}
\resizebox{0.95\textwidth}{!}{
\begin{tabular}{llrrrrrrr}
\toprule
Category & Model & Cum. Ret.$\uparrow$ & Ann. Ret.$\uparrow$ & Sharpe$\uparrow$ & Max DD$\downarrow$ & Vol.$\downarrow$ & Hit Rate$\uparrow$ & Turnover \\
\midrule

\multirow{2}{*}{Passive}
& Market Buy \& Hold
& +11.27\% & +21.53\% & 1.331 & 12.27\% & 15.55\% & \best{57.97\%} & 0.000 \\
& Equal-Weight Buy \& Hold
& +2.89\% & +5.34\% & 0.530 & 10.47\% & 10.94\% & 57.25\% & 0.000 \\
\midrule

\multirow[t]{2}{*}{Tree forecasting}
& LightGBM
& \second{+18.19\%}
& \second{+35.70\%}
& \second{1.723}
& 11.09\%
& 18.74\%
& 54.16\%
& 1.462 \\
& XGBoost
& +10.87\%
& +20.73\%
& 1.280
& 8.51\%
& 15.67\%
& 51.82\%
& 1.442 \\
\midrule

\multirow[t]{5}{*}{Neural forecasting}
& LSTM
& +8.53\% & +16.13\% & 1.007 & 7.44\% & 16.14\% & 52.17\% & 1.042 \\
& GRU
& +7.54\% & +14.19\% & 0.866 & 9.70\% & 16.98\% & 52.17\% & 0.975 \\
& MLP
& +8.07\% & +15.23\% & 0.895 & 12.25\% & 17.56\% & 50.72\% & 1.063 \\
& TRA
& +6.56\% & +12.30\% & 0.737 & 10.20\% & 17.89\% & 52.90\% & 0.963 \\
& HIST
& +9.44\% & +17.91\% & 1.050 & 8.19\% & 17.06\% & 50.72\% & 1.101 \\
\midrule

\multirow[t]{3}{*}{Reinforcement learning}
& FinRL-DQN
& +5.93\% & +11.09\% & 0.840 & 6.32\% & 13.62\% & 51.52\% & 0.445 \\
& FinRL-A2C
& +3.37\% & +6.24\% & 0.560 & 7.20\% & 12.11\% & 49.20\% & 0.274 \\
& FinRL-PPO
& -12.83\% & -22.18\% & -1.456 & 14.16\% & 16.31\% & 46.73\% & 0.410 \\
\midrule

\multirow[t]{3}{*}{Financial LLM}
& FinGPT
& -8.62\% & -15.17\% & -1.457 & 13.93\% & \second{10.88\%} & 45.65\% & 0.014 \\
& FinCast
& -11.73\% & -20.38\% & -1.133 & 16.13\% & 18.57\% & 47.83\% & 1.342 \\
& Kronos
& +2.56\% & +4.73\% & 0.376 & 15.47\% & 15.50\% & 50.72\% & 1.794 \\
\midrule

\multirow[t]{3}{*}{General LLM}
& DeepSeek V4 Pro
& +12.84\% & +24.69\% & 1.493 & 7.01\% & 15.65\% & 53.07\% & 1.050 \\
& DeepSeek V4 Flash
& +7.94\% & +14.98\% & 1.092 & \second{4.67\%} & 13.69\% & 52.77\% & 1.100 \\
& Qwen3.6-35B-A3B
& +6.82\% & +12.80\% & 0.868 & 8.99\% & 15.28\% & 52.98\% & 1.020 \\
\midrule

\multirow[t]{2}{*}{LLM trading agents}
& FinAgent
& +10.21\% & +19.43\% & 1.562 & \best{4.57\%} & 11.81\% & 53.31\% & 1.127 \\
& Trading Agent
& +10.42\% & +19.85\% & 1.353 & 6.07\% & 14.12\% & 51.75\% & 1.214 \\
\midrule

\multirow[t]{4}{*}{External expert-routing}
& LLMoE
& -14.45\% & -24.80\% & -2.285 & 15.99\% & 12.14\% & 47.83\% & 1.306 \\
& FLAG-Trader
& -6.12\% & -10.89\% & -0.624 & 9.32\% & 16.33\% & 43.48\% & 0.794 \\
& TradExpert
& -4.74\% & -8.49\% & -0.899 & 6.03\% & \best{9.38\%} & 49.28\% & 0.444 \\
& MM-DREX
& -2.45\% & -4.43\% & -0.247 & 9.82\% & 14.21\% & 52.17\% & 0.914 \\
\midrule

Ours
& \best{\model{}}
& \best{+49.08\%}
& \best{+107.33\%}
& \best{5.091}
& 5.96\%
& 14.55\%
& \second{57.91\%}
& 1.298 \\
\bottomrule
\end{tabular}
}
\vspace{1mm}
\begin{minipage}{\textwidth}
\footnotesize
The best result is shown in \best{bold}, and the second-best is
\second{underlined}.
\end{minipage}
\end{table*}

\begin{table*}[!t]
\centering
\caption{Portfolio performance on Crypto.}
\label{tab:crypto_main}
\resizebox{0.95\textwidth}{!}{
\begin{tabular}{llrrrrrrr}
\toprule
Category & Model & Cum. Ret.$\uparrow$ & Ann. Ret.$\uparrow$ & Sharpe$\uparrow$ & Max DD$\downarrow$ & Vol.$\downarrow$ & Hit Rate$\uparrow$ & Turnover \\
\midrule

\multirow{2}{*}{Passive}
& Market Buy \& Hold
& -6.38\% & -6.38\% & 0.050 & 32.03\% & 41.72\% & 50.14\% & 0.000 \\
& Equal-Weight Buy \& Hold
& -15.76\% & -15.76\% & 0.055 & 44.97\% & 64.30\% & \second{52.33\%} & 0.000 \\
\midrule

\multirow[t]{2}{*}{Tree forecasting}
& LightGBM
& -2.29\% & -2.29\% & 0.291 & 46.41\% & 66.54\% & 51.29\% & 1.398 \\
& XGBoost
& -0.42\% & -0.42\% & 0.306 & 54.06\% & 62.28\% & 50.58\% & 1.369 \\
\midrule

\multirow[t]{5}{*}{Neural forecasting}
& LSTM
& -24.22\% & -24.22\% & -0.083 & 45.39\% & 67.36\% & 48.30\% & 0.804 \\
& GRU
& +36.08\% & +36.08\% & 0.758 & 37.39\% & 77.24\% & 48.79\% & 0.822 \\
& MLP
& -33.09\% & -33.09\% & -0.317 & 45.81\% & 63.20\% & 49.01\% & 0.738 \\
& TRA
& \second{+43.01\%}
& \second{+43.01\%}
& \second{0.838}
& 39.00\%
& 72.00\%
& 49.07\%
& 0.802 \\
& HIST
& -44.61\% & -44.61\% & -0.551 & 59.39\% & 66.59\% & 50.55\% & 1.191 \\
\midrule

\multirow[t]{3}{*}{RL}
& FinRL-DQN
& -11.47\% & -11.47\% & 0.044 & 36.08\% & 53.70\% & 48.57\% & 0.588 \\
& FinRL-A2C
& -38.91\% & -38.91\% & -1.293 & 45.39\% & \best{33.40\%} & 50.19\% & 0.096 \\
& FinRL-PPO
& -30.39\% & -30.39\% & -0.452 & 40.74\% & 51.44\% & 47.89\% & 0.217 \\
\midrule

\multirow[t]{3}{*}{Financial LMs}
& FinGPT
& +17.75\% & +17.75\% & 0.570 & 41.15\% & 64.88\% & 48.55\% & 0.007 \\
& FinCast
& -34.65\% & -34.65\% & -0.324 & 58.85\% & 65.24\% & 47.40\% & 1.159 \\
& Kronos
& -62.67\% & -62.67\% & -1.487 & 63.75\% & 55.71\% & 49.37\% & 1.500 \\
\midrule

\multirow[t]{3}{*}{General LLMs}
& DeepSeek V4 Pro
& -19.68\% & -19.68\% & -0.065 & 50.16\% & 59.72\% & 48.75\% & 0.832 \\
& DeepSeek V4 Flash
& -30.61\% & -30.61\% & -0.287 & 48.15\% & 61.26\% & 48.07\% & 0.667 \\
& Qwen3.6-35B-A3B
& -58.43\% & -58.43\% & -1.299 & 64.32\% & 55.42\% & 46.32\% & 0.910 \\
\midrule

\multirow[t]{2}{*}{LLM trading agents}
& FinAgent
& -50.56\% & -50.56\% & -0.944 & 58.61\% & 57.11\% & 47.61\% & 1.159 \\
& Trading Agent
& +22.15\% & +22.15\% & 0.638 & 44.15\% & 58.46\% & 50.23\% & 1.195 \\
\midrule

\multirow[t]{4}{*}{External expert-routing}
& LLMoE
& +18.48\% & +18.48\% & 0.625 & \second{31.18\%} & \second{39.60\%} & \best{52.93\%} & 0.434 \\
& FLAG-Trader
& -70.41\% & -70.41\% & -1.043 & 78.64\% & 73.89\% & 52.06\% & 0.504 \\
& TradExpert
& +15.55\% & +15.55\% & 0.537 & 36.96\% & 56.84\% & 49.59\% & 0.409 \\
& MM-DREX
& +5.66\% & +5.66\% & 0.366 & 33.04\% & 51.71\% & \second{52.33\%} & 0.331 \\
\midrule

Ours
& \best{\model{}}
& \best{+73.79\%}
& \best{+73.79\%}
& \best{1.355}
& \best{30.14\%}
& 49.65\%
& 51.36\%
& 1.074 \\
\bottomrule
\end{tabular}
}
\vspace{1mm}
\end{table*}

\subsection{Experimental Setup}\label{sec.exp.setup}

\noindent\textbf{Datasets.}
We evaluate \model{} on two multi-asset daily trading benchmarks.
\emph{Stock} is derived from FNSPID \citep{dong2024fnspid}, a large-scale dataset containing 29.7 million stock-price records and 15.7 million time-aligned financial news articles.
We select 33 U.S. equities spanning 11 sectors and use data from 2021-01-01 to 2023-12-31, which are chronologically split into training, validation, and test sets at a ratio of 7:1:2, thereby preserving temporal order.
\emph{Crypto} is a self-collected benchmark covering a 10 cryptocurrencies.
Its evaluation period runs from 2025-01-01 to 2025-12-31, encompassing both bullish and bearish market phases.
Details of the datasets is provided in Appendix~\ref{app:setup-data}.

\noindent\textbf{Baselines.}
We compare \model{} with 22 baselines from seven families, together with two passive buy and hold benchmarks.
(1) \emph{Tree-based forecasting methods}, including LightGBM\citep{ke2017lightgbm} and XGBo-\\ost\citep{chen2016xgboost}, construct additive ensembles of decision trees that
progressively correct residual errors and capture nonlinear interactions among structured market features.
(2) \emph{Neural forecasting methods}, including LSTM \citep{hochreiter1997long}, GRU \citep{chung2014empirical,cho2014learning}, MLP \citep{hornik1989multilayer}, TRA \citep{lin2021learning}, and HIST \citep{xu2021hist}, learn continuous representations of market data to capture temporal dependencies, cross-asset relations, and changing market patterns for return prediction.
(3) \emph{Reinforcement learning trading methods}, including FinRL-DQN, FinRL-A2C, and FinRL-PPO \citep{liu2021finrl}, learn portfolio policies by
optimizing sequential trading rewards.
(4) \emph{Financial LLM}, including FinGPT \citep{yang2023fingpt}, FinCast \citep{zhu2025fincast}, and Kronos\citep{shi2026kronos},  adapt LLMs to financial text and market context through domain-specific pretraining or fine-tuning.
(5) \emph{General LLM}, including DeepSeek V4 Pro\citep{xu2026deepseek}, DeepSeek V4 Flash\citep{xu2026deepseek}, and Qwen3.6-35B-A3B\citep{team2026qwen3}, directly generate structured trading decisions from the market context without task-specific fine-tuning.
(6) \emph{LLM-based trading agents}, including FinAgent \citep{zhang2024finagent} and Trading Agent \citep{xiao2024tradingagents}, use multi-step reasoning, reflection, memory, or tool-assisted analysis to produce trading decisions.
(7) \emph{External expert-routing methods}, including LLMoE \citep{liu2025llmoe}, FLAG-Trader \citep{xiong2025flagtrader}, TradExpert \citep{ding2024tradexpert}, and MM-DREX \citep{chen2025mmdrex}, coordinate
predefined external experts to guide trading decision generation.

We additionally include two passive strategies:
\emph{Market Buy-and-Hold}, which holds the Nasdaq Composite for the stock benchmark and Bitcoin for the cryptocurrency benchmark, and \emph{Equal-Weight Buy-and-Hold}, which allocates capital
equally across all candidate assets and holds the resulting portfolio.
Detailed baseline configurations are given in Appendix~\ref{app:baseline}.

\noindent\textbf{Evaluation protocol.}
All model training and hyperparameter selection are performed exclusively on the training and validation splits of the Stock benchmark.
The trained models are then evaluated without further training or adaptation on both the chronologically subsequent Stock test set and the held-out Crypto benchmark.
For both markets, trading decisions are processed by the same portfolio simulator with identical transaction-cost and position-limit settings, as detailed in Appendix~\ref{app:setup-sim}.
We report seven metrics, including cumulative return \citep{markowitz1952portfolio}, annualized return \citep{markowitz1952portfolio}, Sharpe ratio \citep{sharpe1966mutual}, maximum drawdown \citep{magdon2004maximum}, annualized volatility \citep{markowitz1952portfolio}, hit rate \citep{pesaran1992simple}, and average daily turnover \citep{demiguel2009optimal}.

\subsection{Performance Evaluation}\label{sec.exp.main}

\stitle{Portfolio performance.}
We report the portfolio performance on Stock and Crypto in Tables~\ref{tab:stock_main} and~\ref{tab:crypto_main} , respectively. We make three observations.
First, \model{} achieves the best cumulative return and Sharpe ratio on both benchmarks.
On Stock, it it obtains a cumulative return of $49.08\%$, outperforming the strongest baseline (LightGBM) by $30.89\%$. On Crypto, it reaches $73.79\%$, exceeding the strongest baseline (TRA) by
$30.78\%$, while also achieving the lowest maximum drawdown.
Second, all external expert-routing methods produce negative returns on Stock, suggesting that routing among manually defined experts may be too coarse to accommodate heterogeneous assets.
Third, \model{} exhibits stronger robustness across markets.
Competitive methods such as FinAgent, LightGBM perform well on Stock but deteriorate markedly on
Crypto, whereas \model{} ranks first on both benchmarks.


\stitle{Cumulative-return dynamics.}
For each benchmark, we further compares \model{} with the three strongest baselines ranked by cumulative return in the corresponding table, together with the market buy-and-hold strategy and show the results in Figure~\ref{fig:cumret_curves}.
We observe that on Stock, \model{} maintains a generally upward trajectory and preserves most
of its accumulated gains toward the end of the evaluation period.
On Crypto, although all methods exhibit substantial fluctuations, \model{}
achieves the highest terminal return with a comparatively limited drawdown.


\subsection{Ablation Study}\label{sec.exp.ablation}

\begin{table*}[t]
\centering
\caption{Ablation study of key components.}
\label{tab:ablation}
\vspace*{-10pt}
\small
\setlength{\tabcolsep}{5pt}
\resizebox{0.7\linewidth}{!}{
\begin{tabular}{lccrrrr}
\toprule
Model
& Q-K Router
& Selection Update
& Cum. Ret.$\uparrow$
& Sharpe$\uparrow$
& Max DD$\downarrow$
& Hit Rate$\uparrow$ \\
\midrule

Variant 1
& $\times$
& $\times$
& +2.83\%
& 0.519
& 8.37\%
& 53.62\% \\

Variant 2
& \checkmark
& $\times$
& +23.39\%
& 2.919
& \best{5.52\%}
& 53.62\% \\

Variant 3
& $\times$
& \checkmark
& +5.26\%
& 0.754
& 7.92\%
& 50.72\% \\

\best{\model{}}
& \checkmark
& \checkmark
& \best{+49.08\%}
& \best{5.091}
& 5.96\%
& \best{57.91\%} \\
\bottomrule
\end{tabular}
}
\vspace{1mm}

\begin{minipage}{0.7\linewidth}
\footnotesize
$\checkmark$ indicates that a component is included, and $\times$
indicates that it is excluded. Q-K router represents the query-key router.
\end{minipage}
\end{table*}

To evaluate the contribution of each component, we compare \model{} with three  ablated variants in Table~\ref{tab:ablation}. We observe that
(1) Variant 2 substantially outperforms Variant 1, showing that the query-key router provides more effective token--expert matching than conventional routing.
(2) Variant 3 brings only moderate improvements over Variant 1, suggesting
that the spare expert selection update alone is insufficient when the underlying routing scores are not sufficiently informative.
(3) The complete model clearly outperforms both single-component variants,
demonstrating that the query-key router and selection update are complementary:
the former improves expert matching, while the latter further refines routing
using feedback from sampled inactive experts.

\subsection{Leakage-Controlled Evaluation}\label{sec.exp.leakage}
LLM-based financial backtests may suffer from temporal leakage  when the
pretraining corpus overlaps with the retrospective evaluation period.
We therefore evaluate \model{} under two leakage-controlled settings.
First, we conduct prospective stock paper trading  from 2026-04-01 to 2026-6-30 using Qwen3.5-9B (released on 2026-03) as backbone.
Second, we repeat the 2025 Crypto evaluation using Qwen2.5-7B (released on 2024-09), whose
checkpoint predates the entire evaluation period.
\begin{table*}[!t]
\centering
\caption{Crypto performance under different backbone temporal settings.}
\label{tab:crypto_leakage}
\setlength{\tabcolsep}{4.5pt}
\resizebox{0.96\textwidth}{!}{
\begin{tabular}{llrrrrrrr}
\toprule
Category
& Model
& Cum. Ret.$\uparrow$
& Ann. Ret.$\uparrow$
& Sharpe$\uparrow$
& Max DD$\downarrow$
& Vol.$\downarrow$
& Hit Rate$\uparrow$
& Turnover \\
\midrule

\multirow{2}{*}{Passive}
& Market Buy \& Hold
& -6.38\%
& -6.38\%
& 0.050
& 32.03\%
& 41.72\%
& 50.14\%
& 0.000 \\
& Equal-Weight Buy \& Hold
& -15.76\% & -15.76\% & 0.055 & 44.97\% & 64.30\% & \second{52.33\%} & 0.000 \\
\midrule
\multirow{5}{*}{Selected baselines}
& TRA
& +43.01\%
& +43.01\%
& 0.838
& 39.00\%
& 72.00\%
& 49.07\%
& 0.802 \\



& LLMoE
& +18.48\%
& +18.48\%
& 0.625
& 31.18\%
& \second{39.60\%}
& \best{52.93\%}
& 0.434 \\

& FinGPT
& +17.75\%
& +17.75\%
& 0.570
& 41.15\%
& 64.88\%
& 48.55\%
& 0.007 \\

\midrule
\multirow{2}{*}{Base LLMs}
& Qwen2.5-7B Base
& -3.86\%
& -3.86\%
& -0.067
& \best{23.69\%}
& \best{22.28\%}
& 48.22\%
& 0.387 \\

& Qwen3.5-9B Base
& -4.82\%
& -4.82\%
& 0.089
& \second{25.35\%}
& 41.30\%
& 50.41\%
& 0.661 \\

\midrule
\multirow{2}{*}{Ours}
& \model{} (Qwen2.5-7B)
& \second{+68.38\%}
& \second{+68.38\%}
& \second{1.259}
& 26.82\%
& 51.88\%
& 50.96\%
& 0.739 \\

& \model{} (Qwen3.5-9B)
& \best{+73.79\%}
& \best{+73.79\%}
& \best{1.355}
& 30.14\%
& 49.65\%
& \second{51.36\%}
& 1.074 \\

\bottomrule
\end{tabular}
}

\end{table*}


We make two observations.
First, under leakage-free evaluation, \model{} consistently outperforms the
selected baselines in both prospective stock paper trading
(Figure~\ref{fig:rolling3m}) and the Qwen2.5-7B-based Crypto evaluation
(Table~\ref{tab:crypto_leakage}).
This confirms that the strong performance of \model{} does not rely on access
to future market outcomes during pretraining.
Second, the potential temporal overlap has only a limited impact.
As shown in Table~\ref{tab:crypto_leakage}, replacing Qwen2.5-7B with
Qwen3.5-9B produces only a modest performance difference, which is much
smaller than the improvement of \model{} over the corresponding base LLMs
and competing methods.
Thus, the overall advantage primarily comes from the proposed method rather
than temporal leakage.

\begin{figure}[ht]
    \centering
    \includegraphics[width=\linewidth]{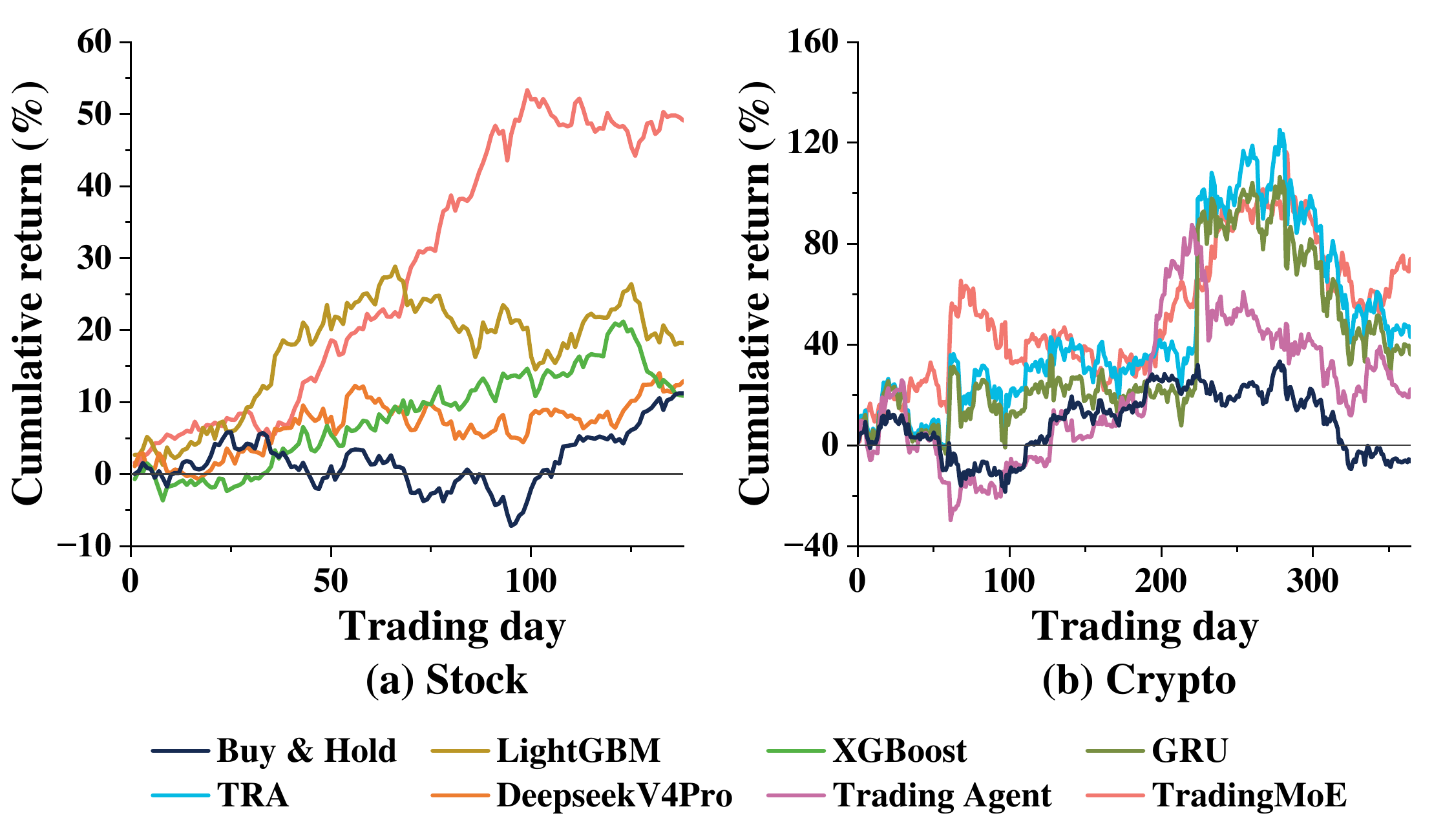}
    \vspace{-25pt}
    \caption{Cumulative return curves on Stock and Crypto.}
    \vspace{-10pt}
    \label{fig:cumret_curves}
\end{figure}

\begin{figure}[t]
    \centering
    \includegraphics[width=\linewidth]{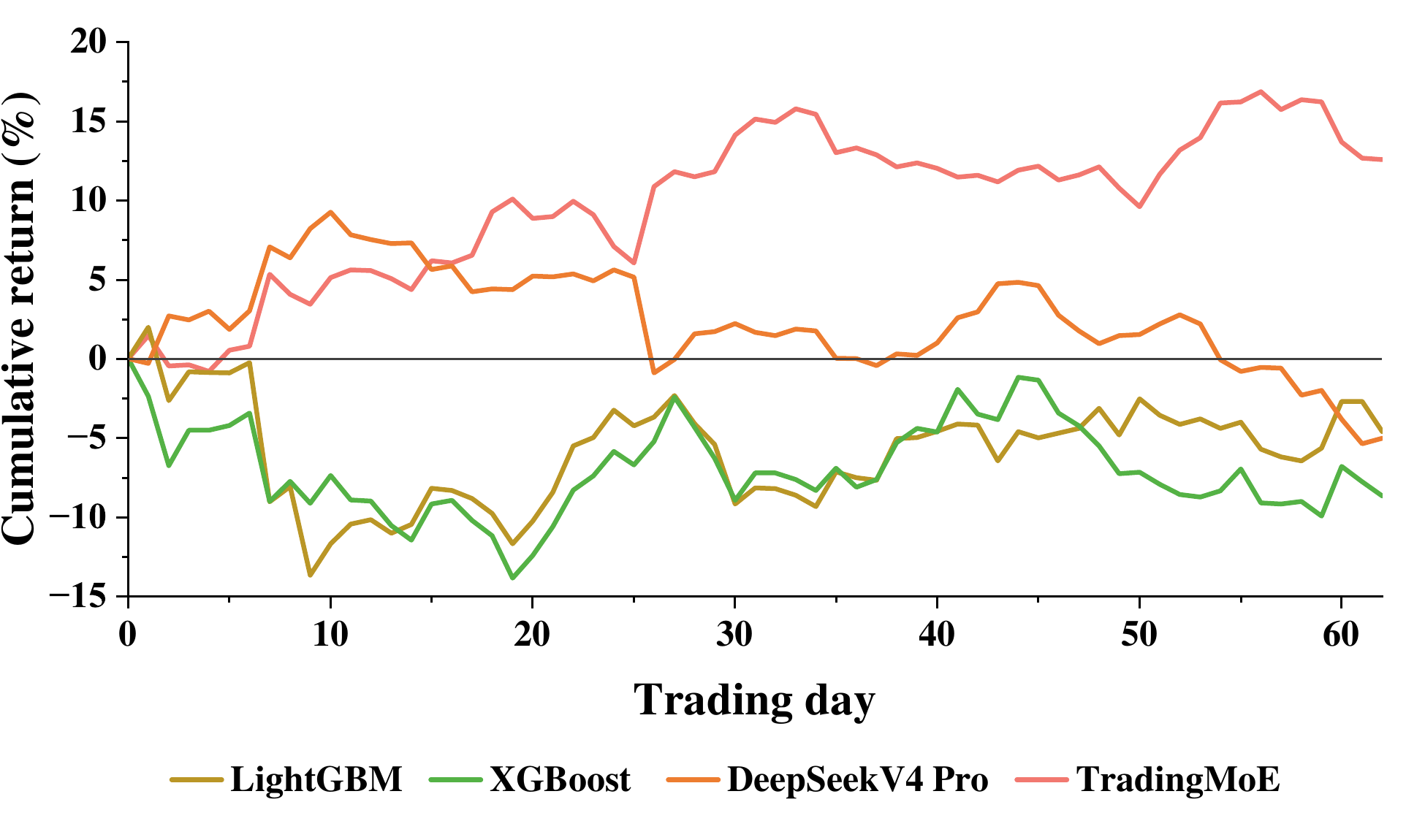}
    \vspace{-30pt}
    \caption{Cumulative return in the paper-trading.}
    \vspace{-10pt}
    \label{fig:rolling3m}
\end{figure}

\subsection{Additional Analyses}
We conduct further analyses, including
hyperparameter sensitivity in Appendix~\ref{sec.exp.qr}, multi-seed
statistical robustness Appendix~\ref{sec.exp.stat}, and transaction-cost
sensitivity Appendix~\ref{sec.exp.cost}.
\section{Conclusions}\label{sec.conclusion}
In this paper, we proposed \model{}, a trading-oriented sparse internal expert routing framework for direct trading-decision generation with LLMs. Motivated by the compact structure of counterfactual expert contribution, \model{} employs a query-key router to match token-specific trading demands with learnable expert representations in a low-dimensional space. Its sparse expert selection update evaluates a small number of inactive experts against the lowest-scored active expert, allowing the router to revise its Top-$k$ selection as market conditions evolve while retaining sparse computation. Extensive experiments on stock and cryptocurrency markets, together with rolling paper trading, demonstrate that \model{} consistently outperforms 22 baselines and generalizes effectively across different market settings.

\clearpage
\newpage

\bibliographystyle{ACM-Reference-Format}
\bibliography{references}

@article{fedus2022switch,
  title={Switch transformers: Scaling to trillion parameter models with simple and efficient sparsity},
  author={Fedus, William and Zoph, Barret and Shazeer, Noam},
  journal={Journal of Machine Learning Research},
  volume={23},
  number={120},
  pages={1--39},
  year={2022}
}

@article{liu2024deepseekv2,
  title={Deepseek-v2: A strong, economical, and efficient mixture-of-experts language model},
  author={Liu, Aixin and Feng, Bei and Wang, Bin and Wang, Bingxuan and Liu, Bo and Zhao, Chenggang and Dengr, Chengqi and Ruan, Chong and Dai, Damai and Guo, Daya and others},
  journal={arXiv preprint arXiv:2405.04434},
  year={2024}
}

@article{markowitz1952portfolio,
  author  = {Harry Markowitz},
  title   = {Portfolio Selection},
  journal = {The Journal of Finance},
  volume  = {7},
  number  = {1},
  pages   = {77--91},
  year    = {1952}
}

@article{demiguel2009optimal,
  title={Optimal versus naive diversification: How inefficient is the 1/N portfolio strategy?},
  author={DeMiguel, Victor and Garlappi, Lorenzo and Uppal, Raman},
  journal={The review of Financial studies},
  volume={22},
  number={5},
  pages={1915--1953},
  year={2009},
  publisher={Oxford University Press}
}

@article{pesaran1992simple,
  title={A simple nonparametric test of predictive performance},
  author={Pesaran, M Hashem and Timmermann, Allan},
  journal={Journal of Business \& Economic Statistics},
  volume={10},
  number={4},
  pages={461--465},
  year={1992},
  publisher={Taylor \& Francis}
}

@article{magdon2004maximum,
  title={On the maximum drawdown of a Brownian motion},
  author={Magdon-Ismail, Malik and Atiya, Amir F and Pratap, Amrit and Abu-Mostafa, Yaser S},
  journal={Journal of applied probability},
  volume={41},
  number={1},
  pages={147--161},
  year={2004},
  publisher={Cambridge University Press}
}

@article{sharpe1966mutual,
  title={Mutual fund performance},
  author={Sharpe, William F},
  journal={The Journal of business},
  volume={39},
  number={1},
  pages={119--138},
  year={1966},
  publisher={JSTOR}
}

@article{tian2024hydralora,
  title={Hydralora: An asymmetric lora architecture for efficient fine-tuning},
  author={Tian, Chunlin and Shi, Zhan and Guo, Zhijiang and Li, Li and Xu, Chengzhong},
  journal={NeurIPS},
  volume={37},
  pages={9565--9584},
  year={2024}
}

@article{hu2022lora,
  title={Lora: Low-rank adaptation of large language models.},
  author={Hu, Edward J and Shen, Yelong and Wallis, Phillip and Allen-Zhu, Zeyuan and Li, Yuanzhi and Wang, Shean and Wang, Liang and Chen, Weizhu and others},
  journal={Iclr},
  volume={1},
  number={2},
  pages={3},
  year={2022}
}

@article{garman1980estimation,
  title={On the estimation of security price volatilities from historical data},
  author={Garman, Mark B and Klass, Michael J},
  journal={Journal of business},
  pages={67--78},
  year={1980},
  publisher={JSTOR}
}

@book{wilder1978new,
  title={New concepts in technical trading systems},
  author={Wilder, J Welles},
  year={1978},
  publisher={Greensboro, NC}
}

@incollection{bridle1990probabilistic,
  title={Probabilistic interpretation of feedforward classification network outputs, with relationships to statistical pattern recognition},
  author={Bridle, John S},
  booktitle={Neurocomputing: Algorithms, architectures and applications},
  pages={227--236},
  year={1990},
  publisher={Springer}
}

@article{eckart1936approximation,
  title={The approximation of one matrix by another of lower rank},
  author={Eckart, Carl and Young, Gale},
  journal={Psychometrika},
  volume={1},
  number={3},
  pages={211--218},
  year={1936},
  publisher={Springer-Verlag}
}

@article{pearson1895vii,
  title={VII. Note on regression and inheritance in the case of two parents},
  author={Pearson, Karl},
  journal={proceedings of the royal society of London},
  volume={58},
  number={347-352},
  pages={240--242},
  year={1895},
  publisher={The Royal Society London}
}

@article{michel2019sixteen,
  title={Are sixteen heads really better than one?},
  author={Michel, Paul and Levy, Omer and Neubig, Graham},
  journal={NeurIPS},
  volume={32},
  year={2019}
}

@inproceedings{molchanov2019importance,
  title={Importance estimation for neural network pruning},
  author={Molchanov, Pavlo and Mallya, Arun and Tyree, Stephen and Frosio, Iuri and Kautz, Jan},
  booktitle={CVPR},
  pages={11264--11272},
  year={2019}
}

@article{politis1994stationary,
  title={The stationary bootstrap},
  author={Politis, Dimitris N and Romano, Joseph P},
  journal={Journal of the American Statistical association},
  volume={89},
  number={428},
  pages={1303--1313},
  year={1994},
  publisher={Taylor \& Francis}
}

@article{ledoit2008robust,
  title={Robust performance hypothesis testing with the Sharpe ratio},
  author={Ledoit, Oliver and Wolf, Michael},
  journal={Journal of Empirical Finance},
  volume={15},
  number={5},
  pages={850--859},
  year={2008},
  publisher={Elsevier}
}

@inproceedings{shi2026kronos,
  title={Kronos: A foundation model for the language of financial markets},
  author={Shi, Yu and Fu, Zongliang and Chen, Shuo and Zhao, Bohan and Xu, Wei and Zhang, Changshui and Li, Jian},
  booktitle={AAAI},
  volume={40},
  number={30},
  pages={25366--25373},
  year={2026}
}

@misc{team2026qwen3,
  title={Qwen3. 6-35B-A3B: Agentic coding power, now open to all},
  author={Team, Qwen},
  year={2026},
  publisher={April}
}

@article{xu2026deepseek,
  title={Deepseek-v4: Towards highly efficient million-token context intelligence},
  author={Xu, Anyi and Lin, Bangcai and Xue, Bing and Wang, Bingxuan and Xu, Bingzheng and Wu, Bochao and Zhang, Bowei and Lin, Chaofan and Dong, Chen and Ling, Chenchen and others},
  journal={arXiv preprint arXiv:2606.19348},
  year={2026}
}

@inproceedings{chen2016xgboost,
  title={Xgboost: A scalable tree boosting system},
  author={Chen, Tianqi and Guestrin, Carlos},
  booktitle={KDD},
  pages={785--794},
  year={2016}
}

@article{ke2017lightgbm,
  title={Lightgbm: A highly efficient gradient boosting decision tree},
  author={Ke, Guolin and Meng, Qi and Finley, Thomas and Wang, Taifeng and Chen, Wei and Ma, Weidong and Ye, Qiwei and Liu, Tie-Yan},
  journal={NeurIPS},
  volume={30},
  year={2017}
}

@article{hornik1989multilayer,
  title={Multilayer feedforward networks are universal approximators},
  author={Hornik, Kurt and Stinchcombe, Maxwell and White, Halbert},
  journal={Neural networks},
  volume={2},
  number={5},
  pages={359--366},
  year={1989},
  publisher={Elsevier}
}

@article{chung2014empirical,
  title={Empirical evaluation of gated recurrent neural networks on sequence modeling},
  author={Chung, Junyoung and Gulcehre, Caglar and Cho, KyungHyun and Bengio, Yoshua},
  journal={arXiv preprint arXiv:1412.3555},
  year={2014}
}

@inproceedings{cho2014learning,
  title={Learning phrase representations using RNN encoder--decoder for statistical machine translation},
  author={Cho, Kyunghyun and Van Merri{\"e}nboer, Bart and Gul{\c{c}}ehre, {\c{C}}a{\u{g}}lar and Bahdanau, Dzmitry and Bougares, Fethi and Schwenk, Holger and Bengio, Yoshua},
  booktitle={EMNLP},
  pages={1724--1734},
  year={2014}
}

@inproceedings{muennighoff2025olmoe,
  title={Olmoe: Open mixture-of-experts language models},
  author={Muennighoff, Niklas and Soldaini, Luca and Groeneveld, Dirk and Lo, Kyle and Morrison, Jacob and Min, Sewon and Shi, Weijia and Walsh, Pete and Tafjord, Oyvind and Lambert, Nathan and others},
  booktitle={ICLR},
  volume={2025},
  pages={62061--62121},
  year={2025}
}

@article{yang2023fingpt,
  title={Fingpt: Open-source financial large language models},
  author={Yang, Hongyang and Liu, Xiao-Yang and Wang, Christina Dan},
  journal={arXiv preprint arXiv:2306.06031},
  year={2023}
}

@article{wu2023bloomberggpt,
  title={Bloomberggpt: A large language model for finance},
  author={Wu, Shijie and Irsoy, Ozan and Lu, Steven and Dabravolski, Vadim and Dredze, Mark and Gehrmann, Sebastian and Kambadur, Prabhanjan and Rosenberg, David and Mann, Gideon},
  journal={arXiv preprint arXiv:2303.17564},
  year={2023}
}

@inproceedings{tang2023finentity,
  title={FinEntity: Entity-level Sentiment Classification for Financial Texts},
  author={Tang, Yixuan and Yang, Yi and Huang, Allen H and Tam, Andy and Tang, Justin Z},
  booktitle={EMNLP},
  pages={15465--15471},
  year={2023}
}

@inproceedings{mukherjee2022ectsum,
  title={Ectsum: A new benchmark dataset for bullet point summarization of long earnings call transcripts},
  author={Mukherjee, Rajdeep and Bohra, Abhinav and Banerjee, Akash and Sharma, Soumya and Hegde, Manjunath and Shaikh, Afreen and Shrivastava, Shivani and Dasgupta, Koustuv and Ganguly, Niloy and Ghosh, Saptarshi and others},
  booktitle={EMNLP},
  pages={10893--10906},
  year={2022}
}

@inproceedings{chen2021finqa,
  title={Finqa: A dataset of numerical reasoning over financial data},
  author={Chen, Zhiyu and Chen, Wenhu and Smiley, Charese and Shah, Sameena and Borova, Iana and Langdon, Dylan and Moussa, Reema and Beane, Matt and Huang, Ting-Hao and Routledge, Bryan R and others},
  booktitle={EMNLP},
  pages={3697--3711},
  year={2021}
}

@inproceedings{chen2022convfinqa,
  title={Convfinqa: Exploring the chain of numerical reasoning in conversational finance question answering},
  author={Chen, Zhiyu and Li, Shiyang and Smiley, Charese and Ma, Zhiqiang and Shah, Sameena and Wang, William Yang},
  booktitle={EMNLP},
  pages={6279--6292},
  year={2022}
}

@inproceedings{li2024master,
  title={Master: Market-guided stock transformer for stock price forecasting},
  author={Li, Tong and Liu, Zhaoyang and Shen, Yanyan and Wang, Xue and Chen, Haokun and Huang, Sen},
  booktitle={AAAI},
  volume={38},
  number={1},
  pages={162--170},
  year={2024}
}

@inproceedings{wang2019alphastock,
  title={Alphastock: A buying-winners-and-selling-losers investment strategy using interpretable deep reinforcement attention networks},
  author={Wang, Jingyuan and Zhang, Yang and Tang, Ke and Wu, Junjie and Xiong, Zhang},
  booktitle={KDD},
  pages={1900--1908},
  year={2019}
}

@inproceedings{ye2020reinforcement,
  title={Reinforcement-learning based portfolio management with augmented asset movement prediction states},
  author={Ye, Yunan and Pei, Hengzhi and Wang, Boxin and Chen, Pin-Yu and Zhu, Yada and Xiao, Ju and Li, Bo},
  booktitle={AAAI},
  volume={34},
  number={01},
  pages={1112--1119},
  year={2020}
}

@inproceedings{wang2021commission,
  title={Commission fee is not enough: A hierarchical reinforced framework for portfolio management},
  author={Wang, Rundong and Wei, Hongxin and An, Bo and Feng, Zhouyan and Yao, Jun},
  booktitle={AAAI},
  volume={35},
  number={1},
  pages={626--633},
  year={2021}
}

@inproceedings{dong2024fnspid,
  title={Fnspid: A comprehensive financial news dataset in time series},
  author={Dong, Zihan and Fan, Xinyu and Peng, Zhiyuan},
  booktitle={KDD},
  pages={4918--4927},
  year={2024}
}

@inproceedings{guo2024finetuning,
  title={Fine-Tuning Large Language Models for Stock Return Prediction Using Newsflow},
  author={Guo, Tian and Hauptmann, Emmanuel},
  booktitle={EMNLP Industry Track},
  pages={1028--1045},
  year={2024}
}

@inproceedings{zhu2025fincast,
  title={FinCast: A Foundation Model for Financial Time-Series Forecasting},
  author={Zhu, Zhuohang and Chen, Haodong and Qu, Qiang and Chung, Vera},
  booktitle={Proceedings of the 34th ACM International Conference on Information and Knowledge Management},
  pages={4539--4549},
  year={2025}
}

@article{xiao2024tradingagents,
  title={TradingAgents: Multi-Agents LLM Financial Trading Framework},
  author={Xiao, Yijia and Sun, Edward and Luo, Di and Wang, Wei},
  journal={arXiv preprint arXiv:2412.20138},
  year={2024}
}

@inproceedings{zhang2024finagent,
  title={A Multimodal Foundation Agent for Financial Trading: Tool-Augmented, Diversified, and Generalist},
  author={Zhang, Wentao and Zhao, Lingxuan and Xia, Haochong and Sun, Shuo and Sun, Jiaze and Qin, Molei and Li, Xinyi and Zhao, Yuqing and Zhao, Yilei and Cai, Xinyu and Zheng, Longtao and Wang, Xinrun and An, Bo},
  booktitle={Proceedings of the 30th ACM SIGKDD Conference on Knowledge Discovery and Data Mining},
  pages={4314--4325},
  year={2024},
  publisher={Association for Computing Machinery},
  doi={10.1145/3637528.3671801}
}

@inproceedings{xiong2025flagtrader,
  title={{FLAG}-{TRADER}: Fusion {LLM}-Agent with Gradient-based Reinforcement Learning for Financial Trading},
  author={Xiong, Guojun and Deng, Zhiyang and Wang, Keyi and Cao, Yupeng and Li, Haohang and Yu, Yangyang and Peng, Xueqing and Lin, Mingquan and Smith, Kaleb E. and Liu, Xiao-Yang and Huang, Jimin and Ananiadou, Sophia and Xie, Qianqian},
  booktitle={Findings of the Association for Computational Linguistics: ACL 2025},
  pages={13921--13934},
  year={2025},
  address={Vienna, Austria},
  publisher={Association for Computational Linguistics},
  doi={10.18653/v1/2025.findings-acl.716}
}

@article{ding2024tradexpert,
  title={Tradexpert: Revolutionizing trading with mixture of expert llms},
  author={Ding, Qianggang and Shi, Haochen and Guo, Jiadong and Liu, Bang},
  journal={arXiv preprint arXiv:2411.00782},
  year={2024}
}

@article{liu2025llmoe,
  title={LLM-Based Routing in Mixture of Experts: A Novel Framework for Trading},
  author={Liu, Kuan-Ming and Lo, Ming-Chih},
  journal={arXiv preprint arXiv:2501.09636},
  year={2025}
}

@article{chen2025mmdrex,
  title={{MM-DREX}: Multimodal-Driven Dynamic Routing of {LLM} Experts for Financial Trading},
  author={Chen, Yang and Jiang, Yueheng and Ma, Zhaozhao and Cao, Yuchen and Keung, Jacky and Kuang, Kun and Gan, Leilei and Wu, Yiquan and Wu, Fei},
  journal={arXiv preprint arXiv:2509.05080},
  year={2025}
}

@inproceedings{shazeer2017outrageously,
  title={Outrageously Large Neural Networks: The Sparsely-Gated Mixture-of-Experts Layer},
  author={Shazeer, Noam and Mirhoseini, Azalia and Maziarz, Krzysztof and Davis, Andy and Le, Quoc and Hinton, Geoffrey and Dean, Jeff},
  booktitle={ICLR},
  year={2017}
}

@inproceedings{lepikhin2020gshard,
  title={{GShard}: Scaling Giant Models with Conditional Computation and Automatic Sharding},
  author={Lepikhin, Dmitry and Lee, HyoukJoong and Xu, Yuanzhong and Chen, Dehao and Firat, Orhan and Huang, Yanping and Krikun, Maxim and Shazeer, Noam and Chen, Zhifeng},
  booktitle={ICLR},
  year={2021}
}

@article{jiang2024mixtral,
  title={Mixtral of experts},
  author={Jiang, Albert Q and Sablayrolles, Alexandre and Roux, Antoine and Mensch, Arthur and Savary, Blanche and Bamford, Chris and Chaplot, Devendra Singh and Casas, Diego de las and Hanna, Emma Bou and Bressand, Florian and others},
  journal={arXiv preprint arXiv:2401.04088},
  year={2024}
}

@article{lo2004adaptive,
  title={The adaptive markets hypothesis: Market efficiency from an evolutionary perspective},
  author={Lo, Andrew W},
  journal={Journal of Portfolio Management, Forthcoming},
  year={2004}
}

@article{ang2012regime,
  title={Regime changes and financial markets},
  author={Ang, Andrew and Timmermann, Allan},
  journal={Annu. Rev. Financ. Econ.},
  volume={4},
  number={1},
  pages={313--337},
  year={2012},
  publisher={Annual Reviews}
}

@inproceedings{yang2019leveraging,
  title={Leveraging {BERT} to Improve the {FEARS} Index for Stock Forecasting},
  author={Yang, Linyi and Dong, Ruihai and Ng, Tin Lok James and Xu, Yang},
  booktitle={Proceedings of the First Workshop on Financial Technology and Natural Language Processing},
  pages={54--60},
  year={2019},
  address={Macao, China}
}

@article{ding2023integrating,
  title={Integrating stock features and global information via large language models for enhanced stock return prediction},
  author={Ding, Yujie and Jia, Shuai and Ma, Tianyi and Mao, Bingcheng and Zhou, Xiuze and Li, Liuliu and Han, Dongming},
  journal={arXiv preprint arXiv:2310.05627},
  year={2023}
}

@inproceedings{wang2024llmfactor,
  title={{LLMFactor}: Extracting Profitable Factors through Prompts for Explainable Stock Movement Prediction},
  author={Wang, Meiyun and Izumi, Kiyoshi and Sakaji, Hiroki},
  booktitle={ACL Findings},
  pages={3120--3131},
  year={2024}
}

@article{xu2021hist,
  title={Hist: A graph-based framework for stock trend forecasting via mining concept-oriented shared information},
  author={Xu, Wentao and Liu, Weiqing and Wang, Lewen and Xia, Yingce and Bian, Jiang and Yin, Jian and Liu, Tie-Yan},
  journal={arXiv preprint arXiv:2110.13716},
  year={2021}
}

@inproceedings{liu2021finrl,
  title={FinRL: Deep reinforcement learning framework to automate trading in quantitative finance},
  author={Liu, Xiao-Yang and Yang, Hongyang and Gao, Jiechao and Wang, Christina Dan},
  booktitle={Proceedings of the second ACM international conference on AI in finance},
  pages={1--9},
  year={2021}
}

@article{feng2019temporal,
  title={Temporal Relational Ranking for Stock Prediction},
  author={Feng, Fuli and He, Xiangnan and Wang, Xiang and Luo, Cheng and Liu, Yiqun and Chua, Tat-Seng},
  journal={TOIS},
  volume={37},
  number={2},
  pages={1--30},
  year={2019},
  publisher={Association for Computing Machinery},
  doi={10.1145/3309547}
}

@article{jiang2017deep,
  title={A deep reinforcement learning framework for the financial portfolio management problem},
  author={Jiang, Zhengyao and Xu, Dixing and Liang, Jinjun},
  journal={arXiv preprint arXiv:1706.10059},
  year={2017}
}

@article{hochreiter1997long,
  title={Long short-term memory},
  author={Hochreiter, Sepp and Schmidhuber, J{\"u}rgen},
  journal={Neural computation},
  year={1997}
}

@inproceedings{lin2021learning,
  title={Learning multiple stock trading patterns with temporal routing adaptor and optimal transport},
  author={Lin, Hengxu and Zhou, Dong and Liu, Weiqing and Bian, Jiang},
  booktitle={Proceedings of the 27th ACM SIGKDD conference on knowledge discovery \& data mining},
  pages={1017--1026},
  year={2021}
}

\clearpage
\appendix
\section{Theoretical Analysis and Routing Diagnostics}
\label{app:method-details}

\subsection{Unbiased Inactive-Expert Sampling}
\label{app:inactive-sampling-proof}

\begin{proof}[Proof of Proposition~\ref{prop:inactive-unbiased}]
Fix token $t$ and routed layer $\ell$. Let
$\mathcal N_t^\ell=(\mathcal I_t^\ell)^c$ denote the inactive expert
set, whose size is $E-K$. For each $i\in\mathcal N_t^\ell$, let
$\mathbf v_{t,i}^\ell$ denote its fixed candidate-wise router-gradient
contribution. The gradient averaged over all inactive experts is
\begin{equation}
\mathbf v_{\mathcal N_t^\ell}
=
\frac{1}{E-K}
\sum_{i\in\mathcal N_t^\ell}
\mathbf v_{t,i}^\ell.
\end{equation}

Let $\hat{\mathcal I}_t^\ell\subset\mathcal N_t^\ell$ be a subset of
size $m$ sampled uniformly without replacement. The sampled estimate is
\begin{equation}
\widehat{\mathbf v}_{\hat{\mathcal I}_t^\ell}
=
\frac{1}{m}
\sum_{i\in\hat{\mathcal I}_t^\ell}
\mathbf v_{t,i}^\ell.
\end{equation}

For each $i\in\mathcal N_t^\ell$, let
$I_i=\mathbf 1[i\in\hat{\mathcal I}_t^\ell]$. Uniform sampling gives
$\mathbb E[I_i]=m/(E-K)$. Therefore,
\begin{equation}
\begin{aligned}
\mathbb E_{\hat{\mathcal I}_t^\ell}
\left[
\widehat{\mathbf v}_{\hat{\mathcal I}_t^\ell}
\right]
&=
\frac{1}{m}
\sum_{i\in\mathcal N_t^\ell}
\mathbb E[I_i]\mathbf v_{t,i}^\ell \\
&=
\frac{1}{E-K}
\sum_{i\in\mathcal N_t^\ell}
\mathbf v_{t,i}^\ell
=
\mathbf v_{\mathcal N_t^\ell}.
\end{aligned}
\end{equation}
Hence, the router-update gradient averaged over the sampled inactive
experts is an unbiased estimate of the gradient averaged over all
inactive experts.
\end{proof}

\subsection{Replacement-Consistent Routing-Margin Update}
\label{app:relative-credit-validity-proof}

\begin{proof}[Proof of Proposition~\ref{prop:routing-margin-update}]
For sampled inactive expert $i\in\hat{\mathcal I}_t^\ell$, recall the
expert-output difference
\begin{equation}
\Delta\mathbf e_{t,i}^\ell
=
\alpha_{t,\mathrm{low}}^\ell
\left[
e_i^\ell(\overline{\mathbf h}_t^\ell)
-
e_{\mathrm{low}}^\ell(\overline{\mathbf h}_t^\ell)
\right].
\end{equation}
Equation~\ref{eq:credit-injection} uses its detached copy and can be
written as
\begin{equation}
\boldsymbol{\xi}_t^\ell
=
\frac{\lambda_{\mathrm{sel}}}{m}
\sum_{j\in\hat{\mathcal I}_t^\ell}
\Delta s_{t,j}^\ell
\mathtt{stopgrad}
\left(
\Delta\mathbf e_{t,j}^\ell
\right).
\end{equation}
Therefore,
\begin{equation}
\frac{\partial\boldsymbol{\xi}_t^\ell}
{\partial\Delta s_{t,i}^\ell}
=
\frac{\lambda_{\mathrm{sel}}}{m}
\mathtt{stopgrad}
\left(
\Delta\mathbf e_{t,i}^\ell
\right).
\end{equation}

Let
\begin{equation}
\mathbf h_{t,0}^\ell
=
\mathtt{FFN}^\ell
\left(
\overline{\mathbf h}_t^\ell
\right)
+
\mathbf r_t^\ell
\end{equation}
denote the routed-layer output before the zero-forward injection.
Equation~\ref{eq:zero-forward-injection} gives
\begin{equation}
\mathbf h_t^\ell
=
\mathbf h_{t,0}^\ell
+
\boldsymbol{\xi}_t^\ell
-
\mathtt{stopgrad}
\left(
\boldsymbol{\xi}_t^\ell
\right).
\end{equation}
Since the stop-gradient term has zero derivative,
\begin{equation}
\left.
\frac{\partial\mathbf h_t^\ell}
{\partial\Delta s_{t,i}^\ell}
\right|_{\mathrm{sel}}
=
\frac{\lambda_{\mathrm{sel}}}{m}
\mathtt{stopgrad}
\left(
\Delta\mathbf e_{t,i}^\ell
\right).
\end{equation}
Applying the chain rule and using the identical forward value of the
detached copy yields
\begin{equation}
\left.
\frac{\partial\mathcal L_{\mathrm{task}}}
{\partial\Delta s_{t,i}^\ell}
\right|_{\mathrm{sel}}
=
\frac{\lambda_{\mathrm{sel}}}{m}
\left\langle
\frac{\partial\mathcal L_{\mathrm{task}}}
{\partial\mathbf h_t^\ell},
\Delta\mathbf e_{t,i}^\ell
\right\rangle.
\end{equation}

We next examine whether the direction of this gradient agrees with
expert-replacement utility. For intervention scale $\rho>0$, define
\begin{equation}
G_{t,i}^\ell(\rho)
=
\mathcal L_{\mathrm{task}}
\left(
\mathbf h_t^\ell
\right)
-
\mathcal L_{\mathrm{task}}
\left(
\mathbf h_t^\ell
+
\rho\Delta\mathbf e_{t,i}^\ell
\right).
\end{equation}
A first-order Taylor expansion around $\mathbf h_t^\ell$ gives
\begin{equation}
G_{t,i}^\ell(\rho)
=
-\rho
\left\langle
\frac{\partial\mathcal L_{\mathrm{task}}}
{\partial\mathbf h_t^\ell},
\Delta\mathbf e_{t,i}^\ell
\right\rangle
+
O
\left(
\rho^2
\left\|
\Delta\mathbf e_{t,i}^\ell
\right\|^2
\right).
\end{equation}

If replacing $e_{\mathrm{low}}^\ell(\cdot)$ with
$e_i^\ell(\cdot)$ is locally expected to reduce the task loss, then
$G_{t,i}^\ell(\rho)>0$ to first order and
\begin{equation}
\left\langle
\frac{\partial\mathcal L_{\mathrm{task}}}
{\partial\mathbf h_t^\ell},
\Delta\mathbf e_{t,i}^\ell
\right\rangle
<0.
\end{equation}
The induced gradient with respect to
$\Delta s_{t,i}^\ell$ is therefore negative, so gradient descent
increases the routing margin. Conversely, a locally harmful replacement
produces a positive gradient and decreases the routing margin. When the
first-order effect is zero, the margin remains unchanged to first order.
\end{proof}

\subsection{Structural Properties of Boundary-Relative Credit}
\label{app:relative-credit-structure}

The inner product appearing in
Proposition~\ref{prop:routing-margin-update} naturally defines the
boundary-relative credit
\begin{equation}
\Delta C_{t,i}^\ell
=
-
\left\langle
\frac{\partial\mathcal L_{\mathrm{task}}}
{\partial\mathbf h_t^\ell},
\Delta\mathbf e_{t,i}^\ell
\right\rangle.
\label{eq:boundary-relative-credit}
\end{equation}
Let the raw credit of expert $j$ be
\begin{equation}
C_{t,j}^\ell
=
-
\left\langle
\frac{\partial\mathcal L_{\mathrm{task}}}
{\partial\mathbf h_t^\ell},
e_j^\ell
\left(
\overline{\mathbf h}_t^\ell
\right)
\right\rangle.
\label{eq:raw-credit-app}
\end{equation}
Substituting the definition of
$\Delta\mathbf e_{t,i}^\ell$ gives
\begin{equation}
\Delta C_{t,i}^\ell
=
\alpha_{t,\mathrm{low}}^\ell
\left(
C_{t,i}^\ell
-
C_{t,\mathrm{low}}^\ell
\right).
\label{eq:relative-raw-credit}
\end{equation}

\noindent\textbf{Ordering preservation.}
For any two candidate experts $i$ and $j$,
\begin{equation}
\Delta C_{t,i}^\ell-\Delta C_{t,j}^\ell
=
\alpha_{t,\mathrm{low}}^\ell
\left(
C_{t,i}^\ell-C_{t,j}^\ell
\right).
\end{equation}
Since $\alpha_{t,\mathrm{low}}^\ell>0$, boundary-relative credit
preserves the within-token ordering of raw expert credit.

\noindent\textbf{Rank bound.}
For structural analysis, extend
$\Delta C_{t,i}^\ell$ to all experts and collect the raw credits across
$T$ target tokens and $E$ experts into
$\mathbf C^\ell\in\mathbb R^{T\times E}$. Let
$\mathbf c^\ell\in\mathbb R^T$ contain the credit of the lowest-scored
active expert for each token:
\begin{equation}
c_t^\ell
=
C_{t,\mathrm{low}_t}^\ell.
\end{equation}
We further define
\begin{equation}
\mathbf D^\ell
=
\operatorname{diag}
\left(
\alpha_{1,\mathrm{low}_1}^\ell,
\ldots,
\alpha_{T,\mathrm{low}_T}^\ell
\right).
\end{equation}
The boundary-relative credit matrix can then be written as
\begin{equation}
\Delta\mathbf C^\ell
=
\mathbf D^\ell
\left(
\mathbf C^\ell
-
\mathbf c^\ell\mathbf 1^\top
\right).
\end{equation}
Since left multiplication cannot increase rank and
$\mathbf c^\ell\mathbf 1^\top$ has rank at most one,
\begin{equation}
\begin{aligned}
\operatorname{rank}
\left(
\Delta\mathbf C^\ell
\right)
&\leq
\operatorname{rank}
\left(
\mathbf C^\ell
-
\mathbf c^\ell\mathbf 1^\top
\right) \\
&\leq
\operatorname{rank}
\left(
\mathbf C^\ell
\right)
+1.
\end{aligned}
\end{equation}
Thus, using the lowest-scored active expert as the reference preserves
expert ordering and introduces at most a rank-one shift to the raw
credit structure.\subsection{Low-Rank Credit Diagnostic}
\label{app:routing-diagnostics}

For this diagnostic, we form a token-by-expert raw credit matrix
$\mathbf C$ at every routed Transformer layer, subtract each row mean,
and compute rank-$r$ SVD reconstructions. For
OLMoE-1B-7B-0125-Instruct, each matrix is $117\times64$ and is
collected from six prompts with Top-$4$ routing. For
DeepSeek-V2-Lite, we use four prompts of at most 96 tokens with 64
routed experts and Top-$6$ routing; shared experts remain on the common
path and are not included as columns of $\mathbf C$.

\emph{Credit Energy Retained} is the fraction of row-centered
Frobenius energy retained. \emph{Higher-Credit Selection Rate} is the
fraction of target tokens for which the Top-$K$ experts selected from
the rank-$r$ reconstruction have higher total credit than those
selected by the naive internal router. Each entry in
Table~\ref{tab:intro_lowrank_credit} is averaged over all routed
layers. SVD fitting and comparison use the same matrix, so this
diagnostic evaluates representational capacity rather than held-out
prediction.

\subsection{Router--Contribution Alignment}
\label{app:routing-utility-alignment}

This diagnostic examines whether internal-router scores reflect the
actual task-loss reduction produced by expert replacement. For each
trading-action or position-size target token $t$ at routed layer
$\ell$, we identify the lowest-scored active expert
$e_{\mathrm{low}}^\ell(\cdot)$ and evaluate every inactive expert
$i\notin\mathcal I_t^\ell$.

The router margin is
$s_{t,i}^\ell-s_{t,\mathrm{low}}^\ell$, while the realized
contribution of candidate $i$ is measured by the finite replacement
gain $G_{t,i}^\ell(1)$ defined in
Proposition~\ref{prop:routing-margin-update}. We compute their Pearson
correlation over all candidate pairs in the held-out validation set,
across all routed layers and multiple training checkpoints. We also
report the fraction of target tokens containing at least one inactive
expert with a positive replacement gain.

The Pearson correlation between router margins and replacement gains is
$-0.015$, while $66.76\%$ of target tokens contain at least one
beneficial inactive expert. These results show that naive
internal-router scores provide little information about which experts
would most reduce the trading-decision loss.

\section{Experimental and Evaluation Details}
\label{app:experimental-protocol}
\subsection{Ground-Truth Decision Construction}
\label{app:ground-truth-construction}
For each trading day, we compute the open-to-close return $r_i$ of every candidate asset and select the five assets with the largest absolute returns. Ties are broken by symbol order. Assets outside the selected set are assigned
\[
y_i=(\texttt{hold},0).
\]

For each selected asset, the trading direction is determined by its return. A return of at least $0.2\%$ is labeled \texttt{long}, a return of at most $-0.2\%$ is labeled \texttt{short}, and a return between these thresholds is labeled \texttt{hold}. For a non-hold decision, the position size $q_i$ is determined by the absolute return, as shown in Table~\ref{tab:ground-truth-size}. A hold decision has $q_i=0$.

\begin{table}[h]
\centering
\caption{Position sizes based on absolute returns.}
\label{tab:ground-truth-size}
\small
\setlength{\tabcolsep}{12pt}
\begin{tabular}{cc}
\toprule
Absolute Return & Position Size $q_i$ \\
\midrule
$0.2\% \leq |r_i| \leq 0.6\%$ & 1 \\
$0.6\% < |r_i| \leq 1.2\%$ & 2 \\
$1.2\% < |r_i| \leq 1.8\%$ & 3 \\
$1.8\% < |r_i| \leq 2.4\%$ & 4 \\
$|r_i| > 2.4\%$ & 5 \\
\bottomrule
\end{tabular}
\end{table}

The resulting action and position size form the ground-truth decision $y_i=(a_i,q_i)$ for each asset, and all asset-level decisions constitute the ground-truth decision set $\mathcal Y=\{y_1,y_2,\ldots,y_M\}$.



\subsection{Data and Causal Preprocessing}\label{app:setup-data}
Table~\ref{tab:dataset-stats} gives the decision dates used by every experiment. Stock is split chronologically and is the only source of parameter updates and model selection; Crypto covers the complete 2025 calendar and is used only for out-of-domain evaluation.

\begin{table}[h]
\centering
\caption{Chronological splits used in the experiments.}
\label{tab:dataset-stats}
\small
\resizebox{\linewidth}{!}{
\begin{tabular}{lllrr}
\toprule
Market & Split & Decision dates (inclusive) & Days & Asset--day rows \\
\midrule
Stock & Train & 2021-03-31--2023-03-03 & 485 & 16,005 \\
Stock & Validation & 2023-03-06--2023-06-12 & 69 & 2,277 \\
Stock & Test & 2023-06-13--2023-12-31 & 138 & 4,554 \\
Crypto & OOD test & 2025-01-01--2025-12-31 & 365 & 3,650 \\
\bottomrule
\end{tabular}
}
\end{table}

\noindent\textbf{Sources and universes.}
Stock prices and news come from FNSPID~\citep{dong2024fnspid}. The universe was fixed before splitting by restricting candidates to U.S.-listed FNSPID equities with valid OHLCV and SEC total-asset metadata, while covering 11 sectors and large-, mid-, and small-size strata where available. It contains \texttt{AIT}, \texttt{AMT}, \texttt{AMZN}, \texttt{APD}, \texttt{AWR}, \texttt{BHP}, \texttt{BKNG}, \texttt{CAT}, \texttt{CMC}, \texttt{COP}, \texttt{CUBE}, \texttt{DUK}, \texttt{ELF}, \texttt{EXEL}, \texttt{FOXA}, \texttt{FSLY}, \texttt{GE}, \texttt{GOOGL}, \texttt{HOG}, \texttt{JNJ}, \texttt{JPM}, \texttt{KHC}, \texttt{MSFT}, \texttt{NRG}, \texttt{NVDA}, \texttt{RDN}, \texttt{RJF}, \texttt{RRC}, \texttt{TGNA}, \texttt{VICI}, \texttt{WMT}, \texttt{XOM}, and \texttt{ZBH}. Crypto combines daily OHLCV files with the CryptoCompare news archive and fixes ten USDT spot pairs before evaluation: \texttt{ADA-USDT}, \texttt{BTC-USDT}, \texttt{DOGE-USDT}, \texttt{ETH-USDT},\\ \texttt{HBAR-USDT}, \texttt{LINK-USDT}, \texttt{LTC-USDT}, \texttt{OKB-USDT}, \texttt{TRX-USDT}, \\and \texttt{XRP-USDT}.

\noindent\textbf{Prices and features.}
FNSPID retains raw OHLCV and adjusted close. Because the task and simulator use same-session open-to-close returns, labels use the contemporaneous raw open and close; adjusted close remains in the source record for corporate-action auditing. Rows without a valid date or close are removed, missing same-row OHLC fields fall back to close, and missing volume is set to zero; prices are never interpolated across dates. Each sample requires five completed bars and uses trailing 30/60-day indicators. Stock windows end at the preceding trading-day close, and Crypto windows end at the preceding UTC daily close.

\noindent\textbf{Causal news.}
Stock news is sourced from the FNSPID news corpus~\citep{dong2024fnspid}, comprising 26,305 unique asset--article records for the selected universe. At 20:00 \texttt{America/New\_York} on the preceding day, the two most recent asset-specific records from the prior seven days are retained. Crypto uses the analogous seven-day window ending at the preceding UTC day boundary, draws from coin-specific and broad-market categories, and retains at most two records by relevance, recency, and sentiment magnitude; the title, or the first 30 words when absent, forms its summary. Empty windows produce an empty news list, and every method reads the same frozen offline records.

\subsection{Unified Portfolio Construction}
\label{app:setup-sim}

All active methods are evaluated through the same decision-conversion,
portfolio-construction, and execution pipeline. For trading day $n$, the
native output of each method is first converted into the canonical
asset-level decision
$y_{n,i}=(a_{n,i},q_{n,i})$
defined in Section~\ref{sec.preliminary}, where
$a_{n,i}\in\{\texttt{long},\texttt{short},\texttt{hold}\}$
denotes the trading action and
$q_{n,i}\in[0,5]$
denotes the position size.

Table~\ref{tab:output-decision-mapping} summarizes the conversion for
different output families. Direct decision generators already produce
$a_{n,i}$ and $q_{n,i}$ and therefore require no additional conversion.
For a return forecaster, we retain the five assets with the largest
absolute predicted returns. The predicted return determines the action,
and its magnitude is converted to a position size using the intervals in
Table~\ref{tab:ground-truth-size}. For an RL policy, the signed action is
clipped to $[-1,1]$; its sign determines the action, and its magnitude
determines the position size. All conversions depend only on model
outputs and do not access realized returns.

\begin{table}[h]
\centering
\caption{Conversion of native outputs to canonical trading decisions.}
\label{tab:output-decision-mapping}
\small
\setlength{\tabcolsep}{5pt}
\resizebox{\linewidth}{!}{
\begin{tabular}{lll}
\toprule
Output family & Native output & Canonical decision \\
\midrule
Return forecaster
& Predicted return $\hat r_{n,i}$
& $a_{n,i}=\mathtt{Dir}(\hat r_{n,i})$,
  $q_{n,i}=Q(|\hat r_{n,i}|)$ \\

Decision generator
& Action $a_{n,i}$, position size $q_{n,i}$
& $(a_{n,i},q_{n,i})$ \\

RL policy
& Signed action $\hat a_{n,i}$
& $a_{n,i}=\mathtt{Dir}(\hat a_{n,i})$,
  $q_{n,i}=5|\operatorname{clip}(\hat a_{n,i},-1,1)|$ \\
\bottomrule
\end{tabular}
}
\end{table}

For return forecasts, $\mathtt{Dir}(z)$ returns \texttt{long} when
$z\geq0.2\%$, \texttt{short} when $z\leq-0.2\%$, and \texttt{hold}
otherwise. For signed RL actions, it returns \texttt{long} when $z>0$,
\texttt{short} when $z<0$, and \texttt{hold} when $z=0$.
$Q(\cdot)$ denotes the return-to-position-size mapping in
Table~\ref{tab:ground-truth-size}.

Let $\mathcal S_n$ denote the resulting set of at most five active
assets. If a method produces more than five active decisions, only the
five with the largest position sizes are retained, with ticker order
used to break ties. The target portfolio weight of asset $i$ is
\begin{equation}
w_{n,i}
=
\begin{cases}
0.2\,\sigma(a_{n,i})\,q_{n,i}/5,
& i\in\mathcal S_n,\\
0,
& \text{otherwise},
\end{cases}
\label{eq:unified-position-postprocessor}
\end{equation}
where
$\sigma(\texttt{long})=1$,
$\sigma(\texttt{short})=-1$, and
$\sigma(\texttt{hold})=0$.
This construction limits each asset to 20\% of the portfolio and
ensures
$\sum_i|w_{n,i}|\leq1$.
Assets that are missing, unselected, or assigned \texttt{hold} receive
zero weight. Unknown assets are discarded, and a response containing
no valid decision produces an empty portfolio.

At the beginning of each trading day, the portfolio is rebalanced from
the previous target weights $\mathbf w_{n-1}$ to the current target
weights $\mathbf w_n$. The resulting turnover is
\begin{equation}
u_n
=
\sum_i
\left|
w_{n,i}-w_{n-1,i}
\right|,
\qquad
\mathbf w_0=\mathbf 0.
\label{eq:portfolio-turnover}
\end{equation}
Thus, an asset that is no longer selected is reduced to zero weight,
while a newly selected asset is adjusted from zero to its target
weight. For an asset retained in the portfolio, transaction costs are
incurred only on the change in its target weight. Reversing a position
from long to short, or from short to long, incurs turnover equal to the
full difference between the two signed weights.

Let $r_{n,i}$ denote the open-to-close return of asset $i$ on day $n$.
The net daily portfolio return is
\begin{equation}
r_n^{\mathrm{net}}
=
\sum_i w_{n,i}r_{n,i}
-
c\,u_n,
\qquad
c=0.0005,
\label{eq:portfolio-accounting}
\end{equation}
where $c$ is the one-way transaction-cost rate of 5 bps. The daily
return is computed using the current target weights $\mathbf w_n$, and
no additional full-liquidation cost is imposed at the end of the day.

Signed weights permit both long and short positions, while the unit
gross-exposure limit prevents leverage. Unallocated capital is treated
as zero-return cash. Borrow fees, additional slippage, and cash interest
are not modeled. The same conversion, rebalancing, and transaction-cost
rules are applied to \model{} and all active baselines. The passive
Buy\&Hold benchmarks bypass this pipeline and are reported only as
references.

\subsection{Metrics}\label{app:setup-metrics}
Let $T$ be the number of dates, $r_t=r_t^{\mathrm{net}}$, $V_0=1$, $V_t=\prod_{s=1}^{t}(1+r_s)$, and $A_M=252$ for Stock and $365$ for Crypto. We report cumulative return $V_T-1$, annualized return $V_T^{A_M/T}-1$, Sharpe ratio $\sqrt{A_M}\,\bar r/\operatorname{std}(r)$ with zero risk-free rate, annualized volatility $\sqrt{A_M}\,\operatorname{std}(r)$, and conventional maximum drawdown. For target weight $w_{t,i}$ and realized same-day open-to-close return $r_{t,i}$, Hit Rate is
\begin{equation}
\operatorname{HitRate}
=
\frac{\sum_{t,i}\mathbf{1}[w_{t,i}r_{t,i}>0]}
{\sum_{t,i}\mathbf{1}[|w_{t,i}|>0,\ |r_{t,i}|>0]}.
\label{eq:directional-hit-rate}
\end{equation}
Thus unselected assets, \texttt{hold} decisions, and zero realized returns do not enter the denominator. Average daily turnover is $T^{-1}\sum_t\tau_t$ and counts both the opening and closing legs without the conventional one-half factor.

\subsection{Statistical Robustness Protocol}\label{app:stat-robust}

We retrain \model{} on Stock with five random seeds, which control adapter and router initialization, shadow sampling, and data ordering; all other settings match Appendix~\ref{app:impl}. Table~\ref{tab:seed-robust} reports per-seed results.

\begin{table}[h]
\centering
\caption{Per-seed results on Stock against LightGBM.}
\label{tab:seed-robust}
\small
\resizebox{\linewidth}{!}{
\begin{tabular}{lrrrrl}
\toprule
Run & Cum. Ret. & Sharpe & LW $p$ & HAC $p$ & Return 95\% CI \\
\midrule
Seed 39 & +50.05\% & 5.366 & 0.026 & 0.020 & [+23.96\%, +84.94\%] \\
Seed 40 & +54.68\% & 5.827 & 0.020 & 0.008 & [+33.04\%, +81.19\%] \\
Seed 41 & +48.09\% & 4.914 & 0.048 & 0.036 & [+21.08\%, +84.09\%] \\
Seed 42 (ref.) & +49.08\% & 5.091 & 0.085 & 0.076 & [+17.32\%, +90.99\%] \\
Seed 43 & +45.72\% & 4.678 & 0.028 & 0.033 & [+23.99\%, +72.58\%] \\
\midrule
Seed average & +49.52\% & --- & 0.009 & 0.004 & [+28.51\%, +75.09\%] \\
\bottomrule
\end{tabular}
}
\parbox{\linewidth}{\footnotesize LW and HAC are one-sided $p$-values against LightGBM; Section~\ref{sec.exp.stat} reports the corresponding two-sided values for the seed average.}
\end{table}

\noindent\textbf{Tests.}
All tests compare daily net return series over the common test dates against LightGBM, the strongest baseline. The Ledoit--Wolf test is the HAC-robust Sharpe-difference test of \citet{ledoit2008robust}; the HAC test is a Newey--West-adjusted $t$-test on the mean daily return difference with automatic lag selection. The seed-average row tests the arithmetic mean of the five daily return series, whose null hypothesis is that the expected daily excess return over LightGBM is non-positive; we do not report a Sharpe point estimate for this averaged series, since cross-seed averaging mechanically reduces volatility and would overstate the Sharpe ratio of a single deployed model.

\noindent\textbf{Confidence intervals.}
Intervals are stationary-block-bootstrap percentile intervals \citep{politis1994stationary} with expected block length $12$ and $10{,}000$ resamples. Under block lengths $10$, $12$, $15$, and $20$, the seed-average return lower bound varies only within $[+28.5\%,+29.1\%]$ and the Sharpe lower bound within $[4.70,4.82]$.

\section{Further Descriptions of Baselines}\label{app:baseline}

We provide additional descriptions of the baseline methods used in
our experiments. All baselines follow the chronological information
cutoff and portfolio-execution protocol described in
Section~\ref{sec.exp}. Methods that require task-specific optimization
are trained only on the Stock training split, while Crypto is retained
as a held-out cross-market benchmark without further training or
adaptation. Each method preserves its original architecture and
learning objective. We standardize only its terminal interface with
the trading environment: methods that directly generate trading
decisions produce the decision JSON described in Section~\ref{sec.preliminary},
whereas methods that output return forecasts or ranking scores are
connected to the same forecast-to-decision conversion and portfolio
simulator. All decisions are evaluated under identical transaction
costs, position constraints, and daily execution frequency.

\noindent\textbf{(1) Passive Benchmark.}
\begin{itemize}
    \item \textbf{Market Buy-and-Hold:}
    A passive strategy that maintains a fixed market exposure
    throughout the evaluation period. We use NASDAQ Buy-and-Hold
    for Stock and BTC Buy-and-Hold for Crypto.

    \item \textbf{Equal-Weight Buy-and-Hold:}
    A passive strategy that allocates capital equally across all
    candidate assets at the beginning of the evaluation period and
    holds the resulting portfolio without active rebalancing.
\end{itemize}

\noindent\textbf{(2) Gradient-Boosted Tree Models.}
Both tree-based methods operate on the same causal tabular market
features and predict future returns, which are subsequently translated
into portfolio decisions using the common execution protocol.

\begin{itemize}
    \item \textbf{LightGBM} \citep{ke2017lightgbm}: A histogram-based
    gradient-boosting framework that grows trees leaf-wise and
    efficiently models nonlinear interactions among price, volume,
    technical, and portfolio-state features.

    \item \textbf{XGBoost} \citep{chen2016xgboost}: An additive
    tree-ensemble model that sequentially fits regression trees to the
    residual errors of the preceding ensemble. Regularization is applied
    to control tree complexity and reduce overfitting.
\end{itemize}

\noindent\textbf{(3) Deep-Learning Predictors.}
These methods learn next-day return signals from structured market
histories. Their return predictions are converted into the common
action and position-size space before portfolio execution.

\begin{itemize}
    \item \textbf{LSTM} \citep{hochreiter1997long}: A recurrent sequence
    model that encodes the historical market trajectory using gated
    memory cells. The final temporal representation is used to predict
    the future return of each asset.

    \item \textbf{GRU} \citep{chung2014empirical,cho2014learning}: A
    recurrent predictor that models temporal dependencies using update
    and reset gates. Compared with LSTM, it adopts a more compact
    recurrent structure while retaining the ability to capture
    long-range market patterns.

    \item \textbf{MLP} \citep{hornik1989multilayer}: A feed-forward
    neural predictor that flattens the rolling market features into a
    fixed-dimensional representation and maps it to asset-level
    future-return estimates. It serves as a basic dense neural baseline
    without an explicit temporal module.

    \item \textbf{TRA} \citep{lin2021learning}: A temporal routing
    architecture that maintains multiple return predictors for different
    latent trading patterns. A routing adaptor assigns each sample to the
    relevant predictors, while an optimal-transport objective encourages
    the predictors to capture distinct temporal regimes.

    \item \textbf{HIST} \citep{xu2021hist}: A relation-aware stock
    predictor that models cross-asset dependencies through
    concept-oriented representations. It aggregates information
    associated with both predefined and latent market concepts to
    generate asset-level return forecasts.
\end{itemize}

\noindent\textbf{(4) Reinforcement-Learning Trading Methods.}
The FinRL variants receive a structured market and portfolio state
and learn trading actions through interaction with the same portfolio
simulator used for evaluation.

\begin{itemize}
    \item \textbf{FinRL-DQN} \citep{liu2021finrl}: A value-based trading
    agent that estimates the expected long-term return of each discrete
    portfolio action using a deep Q-network. The action with the highest
    estimated value is selected at each trading step.

    \item \textbf{FinRL-A2C} \citep{liu2021finrl}: An actor--critic
    trading method in which the actor generates portfolio actions and
    the critic estimates their state values. The advantage signal is used
    to jointly optimize the policy and value networks.

    \item \textbf{FinRL-PPO} \citep{liu2021finrl}: A policy-gradient
    method that updates the trading policy using a clipped surrogate
    objective, limiting excessively large policy changes and improving
    optimization stability in the portfolio environment.
\end{itemize}

\noindent\textbf{(5) Financial Language and Time-Series Models.}
These baselines use financial-domain pre-training to provide return
forecasts or direct trading decisions. Forecasting outputs are passed
through the shared forecast-to-decision conversion before execution.

\begin{itemize}
    \item \textbf{FinGPT} \citep{yang2023fingpt}: A finance-oriented
    language model adapted to financial texts and instruction-following
    tasks. We provide the same structured trading context and use its
    generated response as the final decision under the common JSON
    schema.

    \item \textbf{FinCast} \citep{zhu2025fincast}: A foundation model
    for financial time-series forecasting. It produces asset-level
    future forecasts from historical market sequences, which are
    converted into the standardized trading actions and position sizes.

    \item \textbf{Kronos} \citep{shi2026kronos}: A pre-trained financial
    time-series model that captures temporal patterns from market
    sequences and predicts future price or return dynamics. Its forecasts
    are evaluated through the same portfolio-construction and execution
    pipeline as the other forecasting baselines.
\end{itemize}

\noindent\textbf{(6) Frontier LLM Baselines.}
These models directly receive the structured market-state prompt and
generate the required decision JSON without task-specific fine-tuning.
The same prompt template, temporal cutoff, asset universe, and output
schema are used across all three models.

\begin{itemize}
    \item \textbf{DeepSeek V4 Pro} \citep{xu2026deepseek}: The
    full-capability DeepSeek endpoint is prompted to analyze the supplied
    market context and directly generate predicted returns, actions, and
    position-size levels for the selected assets.

    \item \textbf{DeepSeek V4 Flash} \citep{xu2026deepseek}: A separate
    DeepSeek serving variant evaluated with the same input and output
    protocol as DeepSeek V4 Pro. It serves as an additional
    direct-decision LLM baseline under an identical trading environment.

    \item \textbf{Qwen3.6-35B-A3B} \citep{team2026qwen3}: The official
    Qwen3.6-35B-A3B checkpoint that processes the complete trading
    prompt and directly produces the standardized decision JSON. No
    parameters are updated on the Stock or Crypto benchmark.
\end{itemize}

\noindent\textbf{(7) LLM-Based Trading Agents.}
Unlike direct prompting, these methods retain their original agentic
workflows, including intermediate analysis, reflection, and planning.
Their final recommendations are converted into the common decision
schema and evaluated by the same portfolio simulator.

\begin{itemize}
    \item \textbf{FinAgent} \citep{zhang2024finagent}: A tool-augmented
    financial agent that combines multimodal market information with
    intermediate analysis and reflection. Its agent workflow produces a
    final trading recommendation after integrating the available market
    evidence.

    \item \textbf{Trading Agent} \citep{xiao2024tradingagents}: A
    role-based multi-agent trading framework in which specialized agents
    conduct market analysis, debate candidate strategies, assess risk,
    and pass their conclusions to a final trading agent for decision
    generation.
\end{itemize}

\noindent\textbf{(8) External Expert-Routing Methods.}
These approaches coordinate human-defined external experts at the
model or workflow level. We preserve the expert decomposition and
routing mechanism proposed by each method, while evaluating the final
decision under our common temporal and portfolio protocol.

\begin{itemize}
    \item \textbf{LLMoE} \citep{liu2025llmoe}: An LLM-based
    mixture-of-experts framework in which a model-level router selects
    or combines predefined financial experts according to the current
    market context. The routed expert outputs are aggregated to form the
    final trading recommendation.

    \item \textbf{FLAG-Trader} \citep{xiong2025flagtrader}: A financial
    trading agent that combines LLM-generated market representations
    with gradient-based reinforcement learning. The trainable policy uses
    trading feedback to refine the actions derived from the LLM-agent
    workflow.

    \item \textbf{TradExpert} \citep{ding2024tradexpert}: A trading
    framework composed of specialized LLM experts responsible for
    different analytical roles or information sources. A coordinating
    module integrates their outputs into a unified trading decision.

    \item \textbf{MM-DREX} \citep{chen2025mmdrex}: A multimodal
    dynamic-routing framework that conditions expert selection on
    heterogeneous market signals. Its router dynamically activates or
    weights external LLM experts before aggregating their
    recommendations.
\end{itemize}

The above standardization affects only the terminal decision interface
and portfolio execution. Each baseline retains its original internal
architecture, expert organization, training objective, and inference
workflow.

\section{Implementation Details}\label{app:impl}

\noindent\textbf{Environment.}
The environment in which we run experiments is:
\begin{itemize}
    \item Operating system: Ubuntu 22.04;
    \item GPU information: 16 NVIDIA A800-SXM4-80GB GPUs.
\end{itemize}

\noindent\textbf{Optimizer.}
For \model{}, we use the AdamW optimizer. For all trainable
baselines, we retain the architectures, optimizers, learning-rate
schedules, and hyperparameters specified by their official releases.
We perform no method-specific search on either test set; where an
official procedure selects a checkpoint, selection uses only the
Stock validation split.

\noindent\textbf{Details of baselines.}
For all open-source baselines, we use the officially provided code or
a faithful interface wrapper around it, together with the released
architecture and optimization settings. Only dataset-dependent input
and output dimensions are adapted to the common market features,
asset universe, and trading-decision space; the same portfolio
postprocessor and simulator are then applied to every method.

For FinRL-DQN, FinRL-A2C, and FinRL-PPO
\citep{liu2021finrl}, we use the official FinRL implementations.
FinRL-A2C uses 5 rollout steps, a learning rate of $7\times10^{-4}$,
and an entropy coefficient of $0.01$. FinRL-PPO uses 2,048 rollout
steps, a batch size of 64, a learning rate of $2.5\times10^{-4}$,
and an entropy coefficient of $0.01$. FinRL-DQN follows the default
DQN configuration provided in the official implementation.

For LSTM \citep{hochreiter1997long} and GRU
\citep{chung2014empirical,cho2014learning}, we employ two recurrent
layers with a hidden dimension of 64 and dropout of 0. The models are
trained for at most 200 epochs with a learning rate of $10^{-3}$ and a
batch size of 800. For MLP \citep{hornik1989multilayer}, we use the
official Qlib configuration with a learning rate of $2\times10^{-3}$,
a batch size of 4,096, and at most 8,000 optimization steps. For TRA
\citep{lin2021learning}, we use a two-layer LSTM backbone with a
hidden dimension of 64 and three latent trading states. Its routing
adaptor is implemented as a one-layer LSTM with a hidden dimension
of 32; the sequence length and batch size are set to 60 and 1,024,
respectively. For HIST \citep{xu2021hist}, we employ a two-layer
LSTM backbone with a hidden dimension of 64, use MSE as the
prediction loss, and set the learning rate to $10^{-4}$.

For LightGBM \citep{ke2017lightgbm} and XGBoost
\citep{chen2016xgboost}, we retain the hyperparameters in their
official baseline configurations without task-specific tuning. The
executable configuration files used for all reported runs are
included in the supplementary material.

For Qwen3.6-35B-A3B \citep{team2026qwen3}, DeepSeek V4 Pro
\citep{xu2026deepseek}, and DeepSeek V4 Flash
\citep{xu2026deepseek}, we use the official identifiers
\texttt{Qwen/Qwen3.6-35B-A3B}, \texttt{deepseek-v4-pro}, and
\texttt{deepseek-v4-flash}, respectively, with default reasoning
configurations without task-specific fine-tuning. All three models
receive the same trading prompt and decision format as \model{}.
For the prospective DeepSeek curve in Figure~\ref{fig:rolling3m},
we use the flagship endpoint available at each decision date:
DeepSeek-V3.2 through 2026-04-23 and DeepSeek-V4-Pro from
2026-04-24 onward.

For FinAgent \citep{zhang2024finagent}, we preserve the officially
released market-intelligence, diversified-memory, reflection, and
decision modules. For Trading Agent
\citep{xiao2024tradingagents}, we retain its original analyst roles,
debate procedure, risk-management module, and default model
assignments.

For LLMoE \citep{liu2025llmoe}, we use a five-day lookback window,
a Llama-3.2 router, and two sentiment-oriented experts. Each expert
uses a feed-forward network with hidden dimensions of 128, 64, and
32, with dropout rates of $0.3$ and $0.2$ after the second and third
layers. For FLAG-Trader \citep{xiong2025flagtrader}, we use
SmolLM2-135M-Instruct as the backbone and set the learning rate,
rollout length, discount factor, and PPO clipping coefficient to
$5\times10^{-4}$, 40, $0.95$, and $0.2$, respectively. The model
is trained for 13,860 steps. For TradExpert
\citep{ding2024tradexpert}, we employ four specialized experts and
one general expert based on LLaMA-2-7B. LoRA rank, scaling factor,
and dropout are set to 8, 16, and $0.1$, respectively; the model is
trained for 30 epochs with a learning rate of $10^{-4}$, a batch size
of 4, and 8 gradient-accumulation steps. For MM-DREX
\citep{chen2025mmdrex}, we use Qwen2.5-VL-72B-Instruct with four
expert heads and a 100-day lookback window. Its LoRA rank, scaling
factor, dropout, and learning rate are set to 16, 32, $0.05$, and
$3\times10^{-5}$, respectively; the discount factor and policy
clipping coefficient are set to $0.99$ and $0.15$.

For FinGPT \citep{yang2023fingpt}, we use the officially released
Llama-2-7B LoRA financial-forecaster checkpoint. For FinCast
\citep{zhu2025fincast}, we use its released pre-trained checkpoint
under the recommended zero-shot forecasting setting. For Kronos
\citep{shi2026kronos}, we use the officially released checkpoint and
default inference configuration, with temperature $1.0$, Top-$p$
$0.9$, and one forecast sample.

\noindent\textbf{Details of \model{}.}
For our proposed \model{}, we freeze Qwen3.5-9B and patch every Transformer MLP block with shared-down LoRA experts. We use $r_L{=}12$, $\alpha_L{=}24$, $E{=}64$, $d_q{=}16$, Top-$k{=}4$, and $m{=}2$ sampled inactive experts. The bias-free query head is $\operatorname{Linear}(d,d_q)$--GELU--$\operatorname{Linear}(d_q,d_q)$; expert codes and up-projections are initialized from $\mathcal N(0,0.02^2)$ and $\mathcal N(0,10^{-6})$, respectively, while the remaining linear maps use PyTorch defaults. Router-score temperature is $1.0$, selected-weight temperature is $\tau{=}0.02$, credit scale is $\lambda{=}0.05$, and the RMS floor is $\epsilon{=}10^{-6}$; the balance, z-loss, and entropy coefficients are all zero in the reported runs.

Training uses AdamW with learning rate $10^{-4}$,$(\beta_1,\beta_2)=(0.9$ \\ $,0.999)$, optimizer $\epsilon=10^{-8}$, weight decay $0.01$, no warmup or scheduler, micro-batch size one, one gradient-accumulation step, and gradient clipping at $1.0$. We train for 20,000 updates in \texttt{bfloat16} with DeepSpeed ZeRO-3 and eight-way sequence parallelism, use a 16,000-token training limit and seed 42, and save every 5,000 updates. We select the checkpoint with the highest Stock validation cumulative return (step 20,000, also the highest validation Sharpe) before a single test evaluation; decoding is greedy with a 32,000-token input limit and at most 200 generated tokens. The backbone remains frozen, while the shared down-projections, expert up-projections, query heads, and expert codes are updated; inference removes shadows and credit assignment and executes only ordinary Top-$k$ routing.

One prompt is formed for each decision date and lists the complete candidate-asset universe alphabetically. Each asset entry contains the sector, market capitalization, five completed OHLCV bars, technical indicators, and timestamped news defined in Section~\ref{sec.preliminary}. The model is instructed to generate the canonical asset-level decisions $y_i=(a_i,q_i)$. Unmentioned assets are assigned $(\texttt{hold},0)$ to form the complete decision set $\mathcal Y$.

\section{Additional Experimental Analyses}\label{app:experiment}

\subsection{Hyperparameter Sensitivity}\label{sec.exp.qr}

\begin{figure}[ht]
    \centering
    \includegraphics[width=1\linewidth]{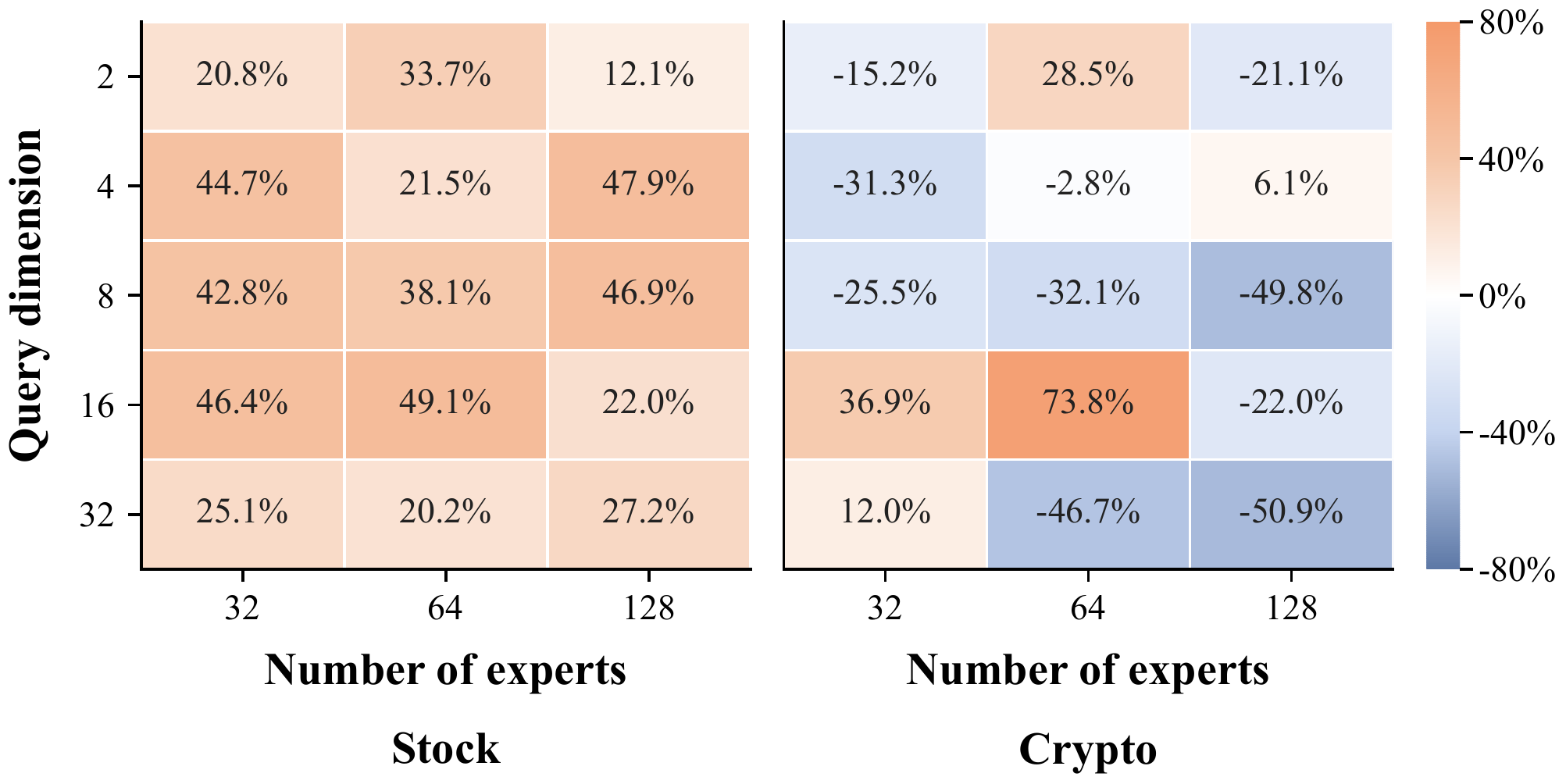}
    \vspace{-20pt}
    \caption{Impact of expert count and query dimension.}
    \label{fig:qr_heatmap}
\end{figure}

We vary the number of each layer experts $E$ and query dimension $d_\vec{q}$ while keeping all other settings fixed, as shown in Fig.~\ref{fig:qr_heatmap}.
We observe that performance is sensitive to their joint configuration: overly large expert
sets or query dimensions often degrade returns, especially on Crypto.
The setting $E=64$ and $d_{\mathrm q}=16$ achieves best performance on
both markets, consistent with the compact token--expert structure identified
in our analysis. 

\subsection{Statistical Robustness}\label{sec.exp.stat}

To verify that the advantage on Stock is not an artifact of a single training run, we retrain \model{} with five random seeds and analyze the resulting daily return series; the full protocol and per-seed results are given in Appendix~\ref{app:stat-robust}.
Across seeds, \model{} attains a cumulative return of $+49.52\%\pm3.30\%$ and a Sharpe ratio of $5.18\pm0.44$ (mean$\pm$std); even the worst seed ($+45.72\%$, Sharpe $4.678$) remains more than $2.5\times$ the strongest baseline, LightGBM ($+18.19\%$, $1.723$).
Testing the seed-averaged daily returns against LightGBM, the Ledoit--Wolf robust Sharpe test rejects equal performance at two-sided $p=0.018$, and a HAC-adjusted mean test at $p=0.007$; per seed, four of five runs are significant at the $5\%$ level and all five at the $10\%$ level.
Stationary-block-bootstrap $95\%$ confidence intervals place the seed-averaged cumulative return in $[+28.5\%,+75.1\%]$ and the Sharpe ratio above $4.73$, so even the lower confidence bounds exceed the strongest baseline's point estimates; these bounds are stable across bootstrap block lengths of $10$--$20$ trading days.
All remaining experiments use the reference seed ($+49.08\%$), which lies within one standard deviation of the five-seed mean.

\subsection{Transaction Cost Sensitivity}\label{sec.exp.cost}

\begin{figure}[ht]
    \centering
    \includegraphics[width=0.8\linewidth]{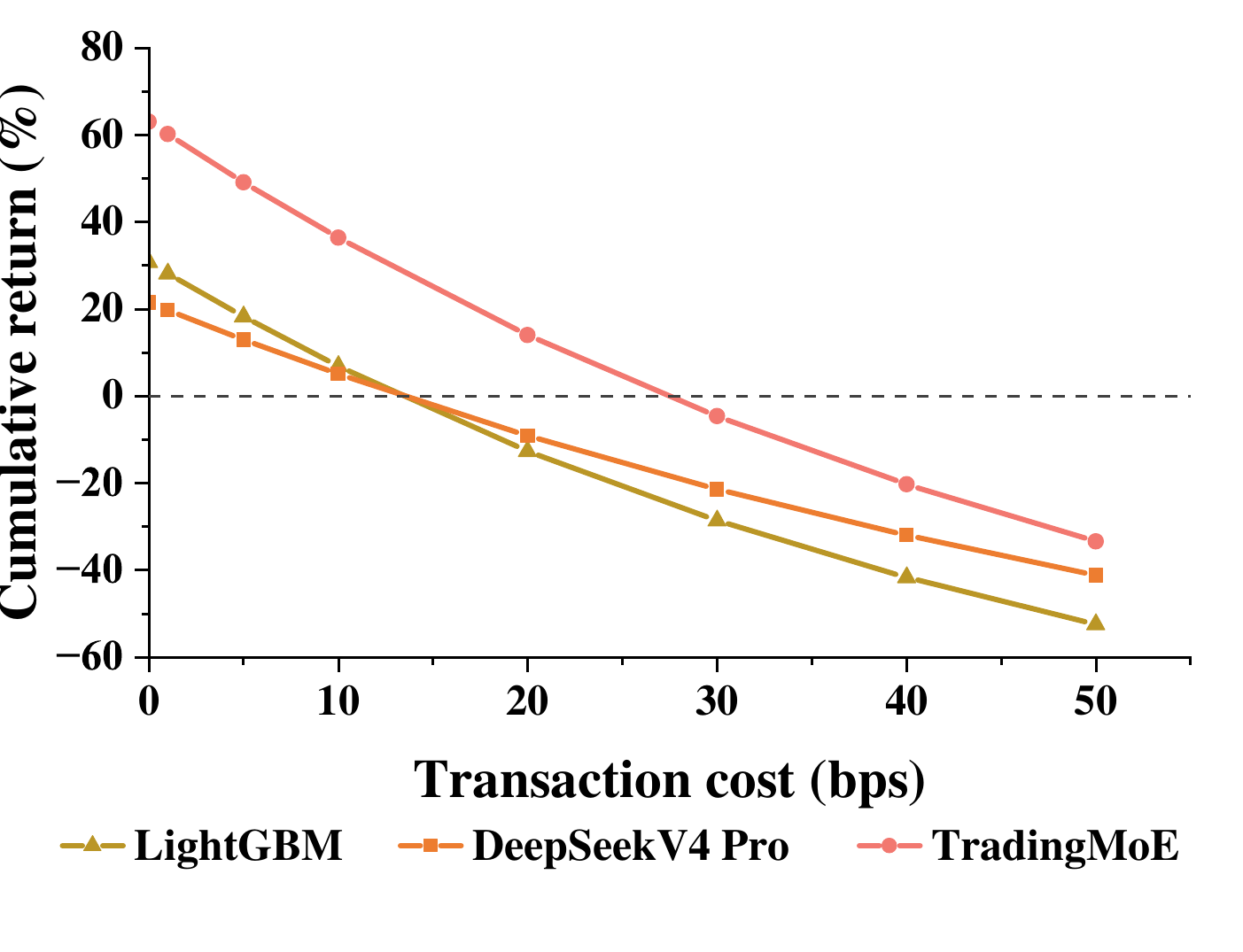}
    \vspace{-20pt}
    \caption{Cumulative return on Stock under varying one-way transaction-cost rates.}
    \vspace{-10pt}
    \label{fig:cost_sensitivity}
\end{figure}

Since \model{} trades with an average daily turnover of $1.298$, we examine how its advantage depends on the assumed cost level by re-running the Stock backtest under one-way transaction-cost rates from $0$ to $50$ bps, compared against the two strongest active baselines, LightGBM and DeepSeek V4 Pro.
Figure~\ref{fig:cost_sensitivity} yields two observations.
First, \model{} degrades no faster than the baselines despite comparable turnover: its return remains positive up to roughly $28$ bps, about twice the break-even cost of LightGBM and DeepSeek V4 Pro (both around $13$--$14$ bps).
Second, the advantage is not an artifact of the default cost setting: even when the cost rate is doubled from the default $5$ bps to $10$ bps, \model{} retains $+36.34\%$, which still exceeds the best baseline evaluated at zero cost ($+30.72\%$).
At extreme rates beyond $30$ bps, all active strategies fall below the passive NASDAQ benchmark, whose return is essentially cost-invariant; in this regime daily rebalancing itself, rather than any particular model, becomes unprofitable.

\end{document}